\documentclass[sigconf]{acmart}
\AtBeginDocument{%
  }

\copyrightyear{2025}
\acmYear{2025}
\setcopyright{acmlicensed}\acmConference[MM '25]{Proceedings of the 33rd ACM International Conference on Multimedia}{October 27--31, 2025}{Dublin, Ireland}
\acmBooktitle{Proceedings of the 33rd ACM International Conference on Multimedia (MM '25), October 27--31, 2025, Dublin, Ireland}
\acmDOI{10.1145/3746027.3755052}
\acmISBN{979-8-4007-2035-2/2025/10}

\usepackage[vlined,boxed,commentsnumbered,linesnumbered,ruled]{algorithm2e}
\usepackage{subfigure}
\usepackage{multirow}
\usepackage{makecell} 
\usepackage[table]{xcolor}

\begin{document}

\title{Target-Guided Bayesian Flow Networks for Quantitatively Constrained CAD Generation}

\author{Wenhao Zheng}
\orcid{0000-0002-4203-0852}
\affiliation{
 \institution{College of Computer Science, Sichuan University} 
 \city{Chengdu} \country{China}
 }
 \email{zhengwenhao@stu.scu.edu.cn}
 \additionalaffiliation{\institution{Engineering Research Center of Machine Learning and Industry Intelligence, Ministry of Education, Chengdu, China}}
 
\author{Chenwei Sun}
\orcid{0009-0004-0111-0030}
\affiliation{
 \institution{College of Computer Science, Sichuan University} 
 \city{Chengdu} \country{China}
 }
\email{sunchenwei@stu.scu.edu.cn}

\author{Wenbo Zhang}
\orcid{0009-0009-2411-1729}
\affiliation{
 \institution{School of Computer Science and Technology, Xidian University} 
 \city{Xi'an} \country{China}
 }
  \email{zhangwenbo01@xidian.edu.cn}

\author{Jiancheng Lv} \authornotemark[1] 
\orcid{0000-0001-6551-3884}
\authornote{Corresponding authors.}
\affiliation{
 \institution{College of Computer Science, Sichuan University} 
 \city{Chengdu} \country{China}
 }
\email{lvjiancheng@scu.edu.cn}

\author{Xianggen Liu} \authornotemark[1] \authornotemark[2]
\orcid{0000-0001-7340-4602}
\affiliation{
 \institution{College of Computer Science, Sichuan University} 
 \city{Chengdu} \country{China}
 }
\email{liuxianggen@scu.edu.cn}

\renewcommand{\shortauthors}{Wenhao Zheng, Chenwei Sun, Wenbo Zhang, Jiancheng Lv, \& Xianggen Liu}

\begin{abstract}
Deep generative models, such as diffusion models, have shown promising progress in image generation and audio generation via simplified continuity assumptions. However, the development of generative modeling techniques for generating multi-modal data, such as parametric CAD sequences, still lags behind due to the challenges in addressing long-range constraints and parameter sensitivity. In this work, we propose a novel framework for quantitatively constrained CAD generation, termed Target-Guided Bayesian Flow Network (TGBFN). For the first time, TGBFN handles the multi-modality of CAD sequences (i.e., discrete commands and continuous parameters) in a unified continuous and differentiable parameter space rather than in the discrete data space. In addition, TGBFN penetrates the parameter update kernel and introduces a guided Bayesian flow to control the CAD properties. To evaluate TGBFN, we construct a new dataset for quantitatively constrained CAD generation. Extensive comparisons across single-condition and multi-condition constrained generation tasks demonstrate that TGBFN achieves state-of-the-art performance in generating high-fidelity, condition-aware CAD sequences. The code is available at \url{https://github.com/scu-zwh/TGBFN}.
\end{abstract}





\maketitle

\section{Introduction}

Computer-Aided Design (CAD) generative modeling plays a critical role in modern design and engineering by facilitating manufacturing, visualization, and data management, advancing contemporary practices in these fields~\cite{jayaraman2021uv,lou2023brep,ritchie2023neurosymbolic}. In CAD design, model shapes are typically represented as sequences of commands or operations (e.g., line, arc, circle) executed by CAD tools~\cite{camba2016parametric,hoffmann2001towards,shah1998designing}. This representation, known as a parametric CAD model, can be treated as a specialized form of text, making it well-suited for deep learning methods~\cite{li2024cad,para2021sketchgen,ganin2021computer,liu2021simulated}. Recent studies have explored various applications of parametric CAD generative modeling~\cite{jones2023self,jones2021automate,willis2022joinable}; however, two inherent challenges hinder the practical application of deep learning to parametric CAD sequence generation: \emph{modality duality} and \emph{parameter sensitivity}. Modality duality arises from the interdependence between discrete categorical command types (e.g., line, arc, circle) and continuous command parameters~\cite{wu2021deepcad,li2024cad,khan2024text2cad}. In contrast, parameter sensitivity refers to the propagation of errors through sequential operations, where even minor deviations (e.g., ±0.1 mm in foundational sketches) can result in catastrophic geometric violations or manufacturing failures~\cite{xu2023hierarchical,fan2022cadtransformer}.

\begin{figure}[t]
    \centering
    \subfigure[Conceptual illustration of the diffusion process]{
        \includegraphics[width=0.8\linewidth]{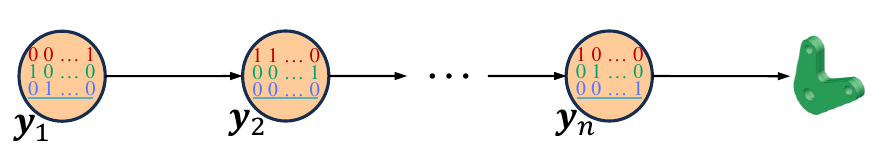}
        \label{fig:diffusion}
    }

    \subfigure[Conceptual illustration of our guided Bayesian inference approach]{
        \includegraphics[width=1.0\linewidth]{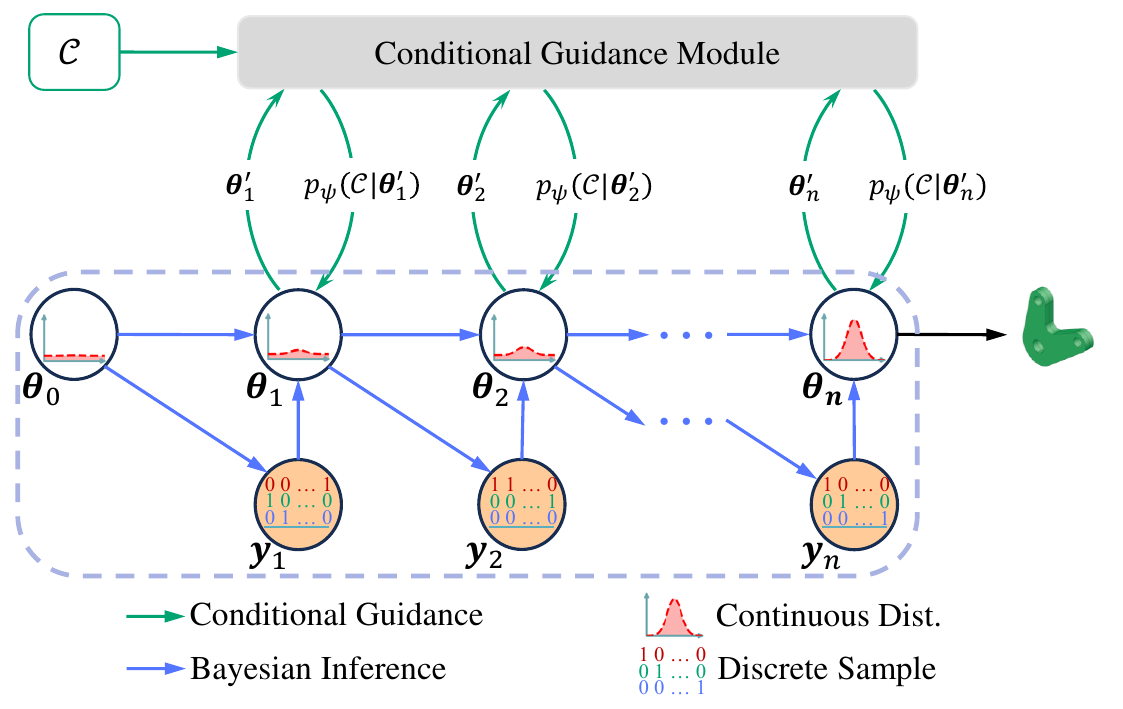}
        \label{fig:method}
    }


    \caption{(a) Conceptual illustration of the diffusion process, which learns a mapping from discrete data samples to a continuous Gaussian prior. (b) Conceptual illustration of our proposed guided Bayesian inference approach. Unlike diffusion models that map samples to a prior, our method directly learns the mapping between the data distribution and a Gaussian distribution. Moreover, it enables precise control over target quantitative CAD properties via a conditional guidance module without necessitating retraining of the Bayesian flow skeleton networks.}

    \vspace{-0.5cm}
    
    \label{fig:gbi}
\end{figure}

Despite these challenges, the field has attracted extensive research efforts~\cite{fan2022cadtransformer,khan2024cad,kodnongbua2023reparamcad}. Autoregressive transformers~\cite{wu2021deepcad,xu2022skexgen,lambourne2022reconstructing} have achieved notable success; however, their sequential prediction mechanisms are prone to error accumulation~\cite{xu2023hierarchical,fan2022cadtransformer}. Recent investigations employing diffusion models (DMs)~\cite{ho2020denoising,dhariwal2021diffusion,rombach2022high} have demonstrated improved reconstruction fidelity through iterative denoising~\cite{ma2024draw,gao2024diffcad,zhang2025diffusion}, yet these approaches face inherent limitations. Specifically, the modality duality between discrete commands and continuous parameters conflicts with the continuous noise assumptions inherent to DMs~\cite{ma2024draw,ho2020denoising}, and the introduction of noise can further exacerbate sensitivity issues~\cite{bai2024regularized}. To address these limitations, CAD-Diffuser~\cite{ma2024draw} proposes a multimodal diffusion framework that resolves modality-specific inconsistencies in unified Gaussian diffusion approaches. At the same time, CADiffusion~\cite{bai2024regularized} develops a regularized decoder architecture that constrains noise-space navigation to enhance output stability. Nevertheless, both approaches exhibit significant drawbacks: the former~\cite{ma2024draw} relies on complex, artifact-prone architectural designs, and the latter~\cite{bai2024regularized} lacks rigorous constraints for effectively defining generative design spaces.

In this paper, we propose a novel framework, coined Target-Guided Bayesian Flow Network (TGBFN), to build a unified continuous parameter space to represent and generate the discrete and continuous elements in CAD sequences (Figure~\ref{fig:gbi}). Specifically, TGBFN first proposes a guided Bayesian flow to control the CAD property by penetrating the conventional parameter update kernel. Then, we introduce an unbiased inference mechanism to eliminate the accumulated errors during the iterative sampling process in BFN. Thirdly, TGBFN adopts a calibrated distribution estimation to approach the theoretically full distribution of data, enhancing the fidelity of generated CAD.

To validate our conditional CAD generative approach, we construct the first CAD dataset with explicit quantitative constraints, using surface area and volume as target conditions. Evaluation on both single-condition and multi-condition constrained generation tasks demonstrates the strong performance of TGBFN.

In summary, our main contributions are as follows:
\begin{itemize}
    \item We propose a novel Bayesian framework that simultaneously models multi-modality CAD sequences within a unified parameter space, establishing a new generative paradigm for precise CAD generation.
    \item The proposed guided Bayesian flow is theoretically guaranteed to be an accurate approximation of the parameter update process under given quantitative conditions.
    \item We also introduce an unbiased inference mechanism and calibrated distribution estimation to further mitigate the exposure bias and enhance the generation quality of CAD sequence.
    \item  Extensive experimental results show that TGBFN significantly outperforms both auto-regressive decoders (e.g., Transformer and DeepCAD) and diffusion-based CAD generators. 
\end{itemize}

\section{Related Work}
\label{sec:relate}

\subsection{Parametric CAD Modeling}

Large-scale parametric CAD datasets~\cite{koch2019abc,wu2021deepcad,willis2021fusion} have significantly accelerated language-based CAD sequence modeling~\cite{guo2022complexgen,jayaraman2022solidgen,wang2022neural}. Research in this field generally falls into two paradigms: unconditional approaches that generate high-quality, diverse 3D models~\cite{wang2020pie,sharma2020parsenet,smirnov2019learning,willis2021engineering} and conditional methods that translate geometric constraints into CAD operations~\cite{ren2022extrudenet,willis2021engineering,xu2021inferring}.

Within conditional generation, input conditions vary widely. Some methods utilize point clouds—as exemplified by DeepCAD~\cite{wu2021deepcad}, while others rely on partial CAD inputs, as in HNC-CAD~\cite{xu2023hierarchical}. Additional approaches incorporate target boundary representations~\cite{willis2021engineering,xu2021inferring} or voxel grids, as employed in SECAD~\cite{lv2021voxel,lambourne2022reconstructing,li2023secad}, or textual prompts~\cite{khan2024text2cad,guo2022complexgen,wang2025text} and image-based inputs~\cite{seff2021vitruvion,chen2024img2cad,you2024img2cad}. Moreover, some techniques use point clouds with or without sequence guidance~\cite{uy2022point2cyl,li2023secad,ren2022extrudenet}. Notably, all these methods depend on \textbf{\emph{qualitative}} input conditions, leaving \textbf{\emph{quantitative}} condition generation largely unexplored. To enable more precise CAD control, we take a key step toward {quantitative conditional generation} by constructing a dedicated dataset and systematically investigating its applications in parametric CAD modeling.

\subsection{Bayesian Flow Networks}
Bayesian Flow Networks (BFNs)~\cite{graves2023bayesian} represent a novel class of generative models that integrate Bayesian inference with flow-based modeling. This framework is particularly appealing because it models the underlying data distribution rather than the data itself—unlike diffusion models~\cite{ho2020denoising,niu2024neural,wu2024diffusion,yu2025knowledge,wu2025enhancing}, which directly model the data through iterative denoising processes. A key advantage of modeling distributions is their inherent continuity and differentiability—properties that are absent in discrete data types (e.g., text)—thereby extending the applicability of BFNs to discrete domains, such as language modeling. Since its inception, the framework has garnered considerable attention \cite{song2023unified,atkinson2025protein,tao2025bayesian}. In particular, \cite{xue2024unifying} theoretically demonstrated that BFNs and diffusion models can be unified from the perspective of stochastic differential equations. Furthermore, GeoBFN \cite{song2023unified} showed that BFNs exhibit less inductive bias than diffusion models, rendering them more suitable for noise-sensitive data.

\section{Preliminary}
\label{sec:pre}

\begin{figure*}[ht]
\centering
    \includegraphics[width=1.0\linewidth]{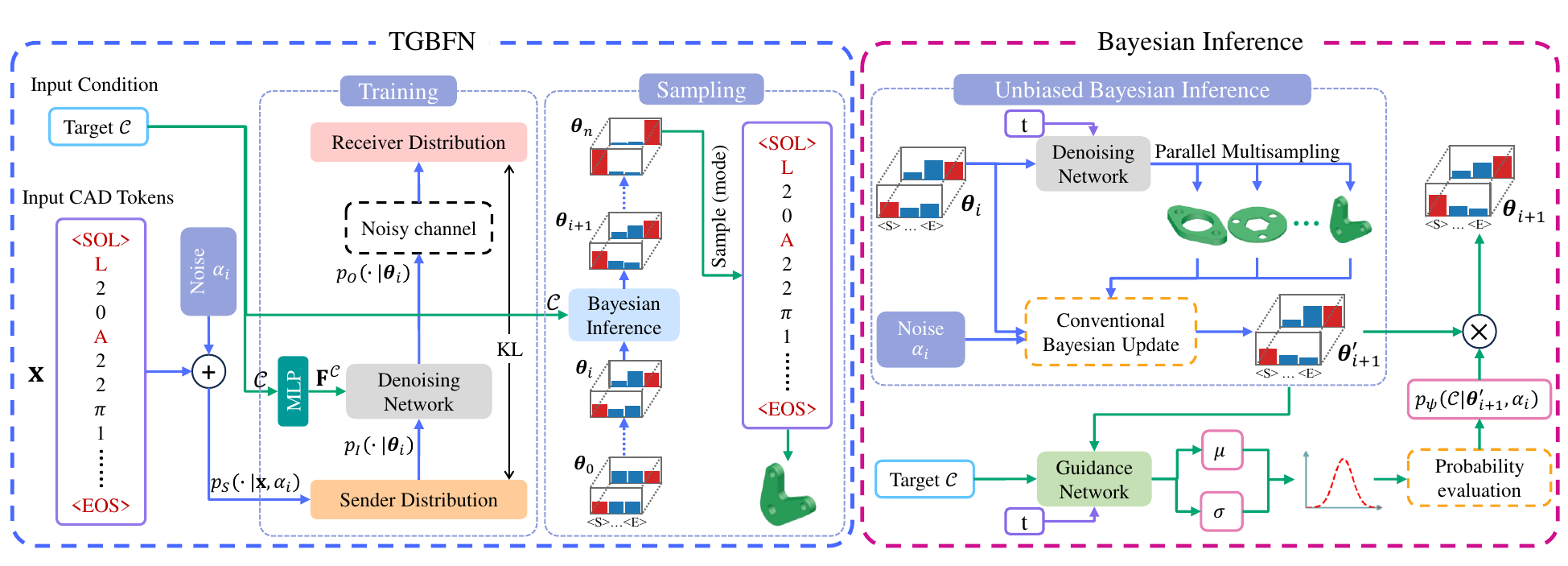}
    \caption{\emph{Method Overview.} The proposed Target-Guided Bayesian Flow Network (TGBFN) implements a Bayesian generative framework for quantitatively constrained parametric CAD sequence generation. During training, the core denoising network and conditional feature extractor (MLP) are optimized by minimizing the KL divergence $\mathcal{D}_{\text{KL}}(p_S \| p_R)$ between the sender and receiver distributions. At inference time, TGBFN performs target-guided sampling using only the desired quantitative condition, enabling accurate and controllable sequence generation. 
    }
    \label{fig:arch}	
\end{figure*}

\subsection{Problem Definition}\label{subsec:prob}

We consider the task of quantitatively constrained CAD generation, where the goal is to generate parametric CAD sequences that satisfy target geometric properties critical to CAD design. Formally, let $\mathbf{x} = (x_1, \dots, x_D) \in \{1, \dots, K\}^D$ denote a parametric CAD token sequence, where $K$ is the vocabulary size and $D$ is the maximum sequence length. Let $\mathcal{C} \in \mathbb{R}_+^d$ represent a vector of $d$ target numerical properties, each taking values in the positive real domain.

The objective is to learn a generator $f: \mathbb{R}_+^d \rightarrow \{1, \dots, K\}^D$ that maps the constraint vector $\mathcal{C}$ to a valid parametric CAD token sequence $\mathbf{x}$ while adhering to the specified quantitative constraints:
\begin{equation}
  \mathbf{x} = f(\mathcal{C}).
\end{equation}

\subsection{Bayesian Flow Networks}
For discrete sequence modeling, Bayesian Flow Networks employ a structured belief state $\boldsymbol{\theta}_i \in \mathbb{R}^{D\times K}$ where each row $\boldsymbol{\theta}_i^{(d)} \in \Delta^{K-1}$ represents categorical distributions for the $d$-th token position. This formulation preserves the discrete nature of symbolic data while enabling continuous belief updates through exponential family parameterization.

The sender distribution for an input sequence $\mathbf{x} \in \{1,\ldots,K\}^D$ generates corrupted embeddings through position-wise Gaussian perturbation:
\begin{equation}
p_S(\mathbf{y}_i\mid\mathbf{x};\alpha_i) = \prod_{d=1}^D \mathcal{N}\left(\mathbf{y}_i^{(d)} \mid \alpha_i(K\mathbf{e}_{x_d}-\mathbf{1}), \alpha_iK\mathbf{I}\right),
\end{equation}
where $\mathbf{e}_{x_d} \in \mathbb{R}^K$ denotes the one-hot encoding of the $d$-th token, and $\mathbf{1}$ is the all-ones vector. The precision schedule $\{\alpha_i\}_{i=1}^n$ controls noise injection intensity, with cumulative precision $\beta_i = \sum_{j=1}^i \alpha_j$ tracking information accumulation.

The Bayesian update distribution implements discrete to continuous mapping through Dirac delta constraints:
\begin{equation}
    p_U(\boldsymbol{\theta}_i\mid\boldsymbol{\theta}_{i-1},\mathbf{x};\alpha)=\sum_{\mathcal{N}(\mathbf{y}\mid\alpha(K\mathbf{e}_{\mathbf{x}}-\mathbf{1}),\alpha K\boldsymbol{I})}\delta\left(\boldsymbol{\theta}-\frac{e^\mathbf{y}\boldsymbol{\theta}_{i-1}}{\sum_{k=1}^Ke_k^\mathbf{y}(\boldsymbol{\theta}_{i-1})_k}\right).
    \label{eq:update}
\end{equation}
This update preserves measure-theoretic consistency while maintaining the belief states in the probability simplex, enabling stable gradient-based learning through subsequent denoising operations. 

Building upon the belief updates, the denoising network $\Phi: \mathbb{R}^{D\times K} \to \mathbb{R}^{D\times K}$ reconstructs clean categorical distributions through time-aware transformations:
\begin{equation}
p_O(\mathbf{x}\mid\boldsymbol{\theta}_i) = \prod_{d=1}^D \mathrm{softmax}\left(\Phi(\boldsymbol{\theta}_i, t_i)^{(d)}\right),
\label{eq:dn}
\end{equation}
where $t_i = (i-1)/n$ and $n \in \mathbb{N}$ denotes the total number of inference steps. Although each $x^{(d)}$ is sampled independently from its corresponding logits $\phi^{(d)}$, these logits are jointly derived from the full belief state $\boldsymbol{\theta}_i$, inducing implicit dependencies across dimensions through the shared network $\Phi$.

The interaction between belief updates and denoising operations is formally captured in the receiver's predictive distribution, which combines positional expectations through Gaussian mixtures:
\begin{equation}
p_R(\mathbf{y}_i\mid\boldsymbol{\theta}_{i-1}) = \prod_{d=1}^D \mathbb{E}_{p_O(x_d\mid\boldsymbol{\theta}_{i-1}^{(d)})}\left[\mathcal{N}\left(\mathbf{y}_i^{(d)}\mid\alpha_i(K\mathbf{e}_{x_d}-\mathbf{1}), \alpha_iK\mathbf{I}\right)\right].
\label{eq:receive}
\end{equation}

These components are unified through the training objective that minimizes the variational bound via position-wise KL divergences~\cite{kullback1951information}:
\begin{equation}
\mathcal{L} = \sum_{i=1}^n \sum_{d=1}^D \gamma_i D_{\mathrm{KL}}\left(p_S(\mathbf{y}_i^{(d)}|x_d;\alpha_i) \parallel p_R(\mathbf{y}_i^{(d)}|\boldsymbol{\theta}_{i-1}^{(d)})\right),
\end{equation}
where $\gamma_i$ denotes the importance weight for the $i$-th step in the parameter refinement process. This systematic framework ensures coordinated optimization of all components while maintaining the theoretical guarantees of Bayesian flow networks.

\section{Methods}
\label{sec:methods}

As illustrated in Figure~\ref{fig:arch}, we propose the Target-Guided Bayesian Flow Network (TGBFN), a novel framework for quantitatively constrained CAD generation. To accomplish the conditional generation task, we reformulate the denoising network (Eq.~\ref{eq:dn}) by incorporating an additional conditional input:
\begin{equation}
p_O(\mathbf{x}\mid\boldsymbol{\theta}_i, \mathcal{C}) = \prod_{d=1}^D \mathrm{softmax}\left(\Phi(\boldsymbol{\theta}_i, t_i, \mathcal{C})^{(d)}\right).
\end{equation}

Despite these modifications, two key challenges remain—exposure bias and precise control of quantitative conditions. To address the challenges, TGBFN incorporates three key components: (1) Unbiased Bayesian Inference, which mitigates exposure bias through parallel sampling (Sec.~\ref{subsec:ubi}); (2) Guided Bayesian Flow, which enables independent control of generation validity and conditional guidance (Sec.~\ref{subsec:cbf}); and (3) Calibrated Distribution Estimation, which preserves statistical fidelity throughout the sampling process (Sec.~\ref{subsec:cdp}). The rest of this section will elaborate on these components.



\subsection{Unbiased Bayesian Inference} \label{subsec:ubi}

Conventional Bayesian Flow Networks (BFNs) accumulate approximation errors during their iterative sampling process, which induces exposure bias and fundamentally limits modeling accuracy. Our unbiased Bayesian inference framework mitigates this error propagation through multi-sample expectation preservation while maintaining real-time performance.

The theoretical Bayesian update (Eq.~\ref{eq:update}) requires traversing the entire value space of the discrete sequence—a process that is computationally intractable. Consequently, practical BFN implementations resort to a single sampling step:
\begin{equation}
\boldsymbol{\theta}_i \leftarrow \frac{e^{\mathbf{y}} \boldsymbol{\theta}_{i-1}}{\sum_k (e^{\mathbf{y}} \boldsymbol{\theta}_{i-1})_k}\quad\text{where}\quad \mathbf{y} \sim \mathcal{N}(\alpha(K\mathbf{e_x}-1), \alpha K\mathbf{I}),
\end{equation}
where the distribution parameters $\boldsymbol{\theta}=(\theta^{1},...,\theta^{D})\in[0,1]^{KD}$ define categorical distributions across $D$ discrete variables, with each $\theta^{d}=(\theta^{d}_1,...,\theta^{d}_K)\in\Delta^{K-1}$ representing the parameters of a $K$-class categorical distribution for variable $d$. Here $\theta^{d}_k\in[0,1]$ specifies the probability of variable $d$ taking class $k$, satisfying $\sum_{k=1}^K\theta^{d}_k=1$. The composite embedding $\mathbf{e_x}\stackrel{\text{def}}{=}(\mathbf{e}_{x^{(1)}},...,\mathbf{e}_{x^{(D)}})\in\mathbb{R}^{KD}$ aggregates one-hot encodings where each $\mathbf{e}_{x^{(d)}}\in\{0,1\}^K$ is a basis vector with $(\mathbf{e}_{x^{(d)}})_k=\delta_{jk}$ when $x^{(d)}=j$, using Kronecker delta notation $\delta_{jk}=1$ if $j=k$ else $0$.

Our multi-sample Bayesian update mechanism mitigates this exposure bias through parallel uncertainty propagation. By executing $m$ independent sampling paths, we update the distribution as follows:
\begin{equation}
\boldsymbol{\theta}_i \leftarrow \frac{1}{m}\sum_{j=1}^m \frac{e^{\mathbf{y}_j} \boldsymbol{\theta}_{i-1}}{\sum_k (e^{\mathbf{y}_j} \boldsymbol{\theta}_{i-1})_k}\quad \text{where} \quad \mathbf{y}_j \sim \mathcal{N}(\alpha(K\mathbf{e_x} - \mathbf{1}), \alpha K\mathbf{I}),
\end{equation}
where $m$ controls the number of parallel samples, and $\mathbf{1} \in \mathbb{R}^K$ denotes the all-ones vector. Experimental results show that increasing \( m \) improves MSE, MAE, and PCC, indicating more accurate and stable inference. These findings suggest that minimal parallelism is both effective and efficient for mitigating exposure bias.



\subsection{Guided Bayesian Flow}\label{subsec:cbf}

While Bayesian Flow Networks (BFNs) offer theoretical advantages for modeling distributional parameters, they lack an effective mechanism for incorporating conditional constraints during generation. To address this limitation, we introduce a principled conditional generation framework that decomposes the Bayesian inference process. Theorem~\ref{cor:cbu} provides the theoretical foundation by factorizing the conditional Bayesian update into two independent components: a standard Bayesian update and a condition-aware guidance term while maintaining full probabilistic consistency.

\begin{theorem}[Proof in Appendix]\label{cor:cbu}
Let \(\mathcal{C} \in \mathbb{R}_+^d\) denote target numerical conditions, and let $\mathbf{x} \in \mathcal{X}$ represent a parametric CAD token sequence. If $\mathcal{C}$ and $\mathbf{x}$ are conditionally independent given the current distribution parameters $\boldsymbol{\theta}_i$, i.e.,
\begin{equation}
    p(\boldsymbol{\theta}_{i-1},\mathbf{x},\alpha,\mathcal{C}\mid\boldsymbol{\theta}_{i}) = p(\boldsymbol{\theta}_{i-1},\mathbf{x},\alpha\mid\boldsymbol{\theta}_{i})p(\mathcal{C}\mid\boldsymbol{\theta}_{i}),
\end{equation}
then the conditional Bayesian update admits the following decomposition:
\begin{equation}
p(\boldsymbol{\theta}_i\mid\boldsymbol{\theta}_{i-1},\mathbf{x},\alpha,\mathcal{C}) \propto \underbrace{p(\boldsymbol{\theta}_i\mid\boldsymbol{\theta}_{i-1},\mathbf{x},\alpha)}_{\text{Standard Update}}\quad \cdot \underbrace{p(\mathcal{C}\mid\boldsymbol{\theta}_i, \alpha),}_{\text{Conditional Guidance}}
\end{equation}
which separates the intrinsic BFN dynamics from external condition-specific constraints.
\end{theorem}

This decomposition defined in Theorem 4.1 leads to a dual-objective formulation: the standard Bayesian term preserves geometric fidelity, while the guidance term encourages alignment with the desired target condition. $p(\mathcal{C}\mid\boldsymbol{\theta}_i, \alpha)$ stands for the posterior probability that the CAD sequence $\mathbf{x}$ denoised from $\boldsymbol{\theta}_i$ has the property $\mathcal{C}$, which is intractable in 
theory. 

To model the conditional guidance distribution \( p(\mathcal{C} \mid \boldsymbol{\theta}_i, \alpha) \), we assume a Gaussian form parameterized by a mean-variance prediction network:
\begin{equation}
p_{\psi}(\mathcal{C} \mid \boldsymbol{\theta}_i, \alpha) := \mathcal{N}\left(\mathcal{C}; \mu_{\psi}(\boldsymbol{\theta}_i, \alpha), \Sigma_{\psi}(\boldsymbol{\theta}_i, \alpha)\right),
\end{equation}
where \( \mu_{\psi} \) and \( \Sigma_{\psi} \) are neural networks that predict the mean and covariance of the target distribution. A critical distinction is that the guidance network's input is not a specific parametric CAD sequence sample \(\mathbf{x}\) itself, but rather the distribution parameter \(\boldsymbol{\theta}_i\). Since the ground-truth mean and variance corresponding to distribution \(\boldsymbol{\theta}_i\) are not directly observable, we estimate them by leveraging the theoretical formulation presented in Theorem~\ref{thm:p2}. This derivation provides a principled foundation for constructing the supervisory signals required to train the conditional guidance network.

\begin{theorem}[Proof in Appendix]\label{thm:p2}
Let $\boldsymbol{\theta}_i$ denote the distribution parameters at the $i$-th step, $\alpha$ an accuracy parameter, and $(\mathbf{x},\mathcal{C})$ a data pair drawn from the dataset $\mathcal{D}$. If the conditional distribution $p(\mathcal{C}\mid\boldsymbol{\theta}_i, \alpha)$ follows a Gaussian distribution $\mathcal{N}(\mu(\boldsymbol{\theta}_i, \alpha),\Sigma(\boldsymbol{\theta}_i, \alpha))$, we have:
\begin{subequations}
\begin{align}
\mu(\boldsymbol{\theta}_i, \alpha) &= \frac{1}{Z} \sum_{(\mathbf{x}, \mathcal{C}) \sim \mathcal{D}} p_S(\boldsymbol{\theta}_i \mid \mathbf{x}, \alpha) \cdot \mathcal{C}, \\
\Sigma(\boldsymbol{\theta}_i, \alpha) &= \frac{1}{Z} \sum_{(\mathbf{x}, \mathcal{C}) \sim \mathcal{D}} p_S(\boldsymbol{\theta}_i \mid \mathbf{x}, \alpha) \left[ \mathcal{C} - \mu(\boldsymbol{\theta}_i, \alpha) \right]^2,
\end{align}
\end{subequations}
where $Z = \sum_{(\mathbf{x}, \mathcal{C}) \sim \mathcal{D}} p_S(\boldsymbol{\theta}_i \mid \mathbf{x}, \alpha)$ is the normalization constant.
\end{theorem}

Based on Theorem~\ref{thm:p2}, we generate supervisory annotations for the conditional guidance network by estimating the mean and variance associated with the distribution $\boldsymbol{\theta}_i$. As computing expectations over the entire dataset $\mathcal{D}$ is computationally intensive, we approximate $\mu$ and $\Sigma$ using mini-batch statistics during training. These estimates enable supervised learning of $\mu_{\psi}$ and $\Sigma_{\psi}$ via the following KL-divergence objective:
\begin{equation}
\mathcal{L}_{\text{gbf}} = D_{\text{KL}}\left(p_{\psi}(\mathcal{C} \mid \boldsymbol{\theta}_i, \alpha) \;\big\|\; \hat{p}(\mathcal{C} \mid \boldsymbol{\theta}_i, \alpha)\right),
\end{equation}
where \(\hat{p}\) denotes the empirical Gaussian distribution parameterized by \(\hat{\mu}\) and \(\hat{\Sigma}\). This objective encourages statistical alignment between the learned distribution and the data-driven empirical behavior.

Building on Theorem~\ref{cor:cbu}, we integrate the conditional guidance into the Bayesian update as:
\begin{equation}
\boldsymbol{\theta}_i \leftarrow \frac{e^{\mathbf{y}} \boldsymbol{\theta}_{i-1}}{\sum_k (e^{\mathbf{y}} \boldsymbol{\theta}_{i-1})_k} \cdot \mathcal{N}\left(\mathcal{C}; \mu_{\psi}(\boldsymbol{\theta}_i, \alpha), \Sigma_{\psi}(\boldsymbol{\theta}_i, \alpha)\right),
\end{equation}
where the multiplicative term enforces consistency with the target conditions. The full sampling algorithm is provided in Algorithm~\ref{alg:gbf}. Unlike standard BFN sampling (Algorithm~\ref{alg:bfn} in Appendix), our method introduces an additional condition guidance step (line~\ref{eq:cg}) to refine distribution parameters in accordance with desired constraints.

\setlength{\textfloatsep}{5pt}
\begin{algorithm}[t]
\caption {GBF Sampling}
\label{alg:gbf}
\SetKwInOut{KwInput}{Require}
\KwInput{
target conditions $\mathcal{C}$, $\beta(1) \in \mathbb{R}_+$, number of steps $n \in \mathbb{N}$, number of classes $K \in \mathbb{N}$, the trained conditional guidance network $p_{\psi}$
}
$\boldsymbol{\theta}_0 \leftarrow ( \frac{\mathbf{1}}{\mathbf{K}} )$\;
\For{$i=1$ \textbf{to} $n$}{
    $t \leftarrow \frac{i-1}{n}$\;
    $\mathbf{k} \sim \text{OUTPUT\_DISTRIBUTION}(\boldsymbol{\theta}_{i-1},t,\mathcal{C})$\;
    $\alpha\sim\beta(1)(\frac{2i-1}{n^2})$\;
    $\mathbf{y}\sim\mathcal{N}(\alpha(K\mathbf{e_k-1}), \alpha K\mathbf{I})$\;
    $\boldsymbol{\theta}_i^{'}\leftarrow \frac{e^\mathbf{y}\boldsymbol{\theta}_{i-1}}{\sum_k(e^\mathbf{y}\boldsymbol{\theta}_{i-1})_k}$\;
    $\boldsymbol{\theta}_i\leftarrow\boldsymbol{\theta}_i^{'} \cdot p_{\psi}(\mathcal{C}\mid\boldsymbol{\theta}_i^{'},\alpha)$  \label{eq:cg}
    }
$\mathbf{k} \sim \text{OUTPUT\_DISTRIBUTION}(\boldsymbol{\theta}_n,t,\mathcal{C})$\;
\Return{$\mathbf{k}$}\;
\end{algorithm}



\subsection{Calibrated Distribution Estimation} \label{subsec:cdp}

Conventional Bayesian Flow Networks (BFNs) suffer from fidelity degradation due to coarse distribution approximations during the sampling process~\cite{song2023unified}. Although the theoretical formulation shows that the sender distribution can be derived by full integration over discrete output distributions (Eq.~\ref{eq:receive}), exact computation is infeasible due to the exponential growth of the combinatorial space. To enable tractable inference, practical implementations resort to discrete sampling schemes as approximations to the full integration. The standard BFN sampling process is given by:
\begin{gather}
    \mathbf{k} \sim p_O(\cdot \mid \boldsymbol{\theta}, t), \\
    \mathbf{y} \sim \mathcal{N}(\alpha(K\mathbf{e_k} - \mathbf{1}), \alpha K \mathbf{I}),
\end{gather}
where $\mathbf{k}$ is the output distribution sample and $\mathbf{y}$ is the sender distribution sample. While computationally efficient, this approach introduces systematic bias due to coarse quantization and limited sampling resolution (Algorithm~\ref{alg:bfn} in Appendix).

To address this mismatch, we propose a calibrated expectation-projection mechanism that preserves statistical moments through quantized averaging:
\begin{gather}
    {\mathbf{k}}'(\boldsymbol{\theta}, t) = \mathrm{NEAREST\_CATEGORY}\left( \frac{1}{H} \sum_{h=1}^{H} \mathbf{k}_h(\boldsymbol{\theta}, t) \right), \\
    \mathbf{k}_h \sim p_O(\cdot \mid \boldsymbol{\theta}, t),\\
    \mathbf{y} \sim \mathcal{N}(\alpha(K\mathbf{e_{{k}'}} - \mathbf{1}), \alpha K \mathbf{I}),
\end{gather}
where $\mathrm{NEAREST\_CATEGORY}: \mathbb{R}^D \rightarrow \mathcal{K}^D$ maps the averaged vector to the nearest discrete category, with $\mathcal{K} = \{1, \dots, K\}$ and $H$ denoting the {sampling granularity}, i.e., the number of repeated draws used to calibrate the categorical expectation. This projection is applied dimension-wise:
\begin{equation}
    \mathrm{NEAREST\_CATEGORY}(z^{(d)}) = \mathop{\mathrm{argmin}}\limits_{k \in \mathcal{K}} |z^{(d)} - k|.
\end{equation}

This calibrated strategy reduces bias by aligning the sampled distribution with its target while preserving sampling efficiency. As shown in Figure~\ref{fig:ablation_m}, increasing sampling granularity $H$ improves performance across all metrics, with diminishing returns beyond moderate values—indicating that low granularity is sufficient for effective calibration.


\subsection{Training Procedure} \label{subsec:training}

Our framework consists of two independently trained components: the skeleton Bayesian flow network and the conditional guidance network. The skeleton network is trained to learn the underlying generative structure of CAD sequences in an unconditional manner, capturing the statistical patterns and structural priors of valid geometric programs. In parallel, the conditional guidance network (Section~\ref{subsec:cbf}) is trained to align generated samples with target geometric properties by minimizing the divergence between predicted and empirical distributions. This modular training strategy enables stable optimization and supports precise constraint control during inference.

\section{Experiments}
\label{sec:exp}

\subsection{Setup}

\subsubsection{Dataset Construction} \label{subsec:dataset}

\begin{figure*}[htbp]
\centering
    \includegraphics[width=\linewidth]{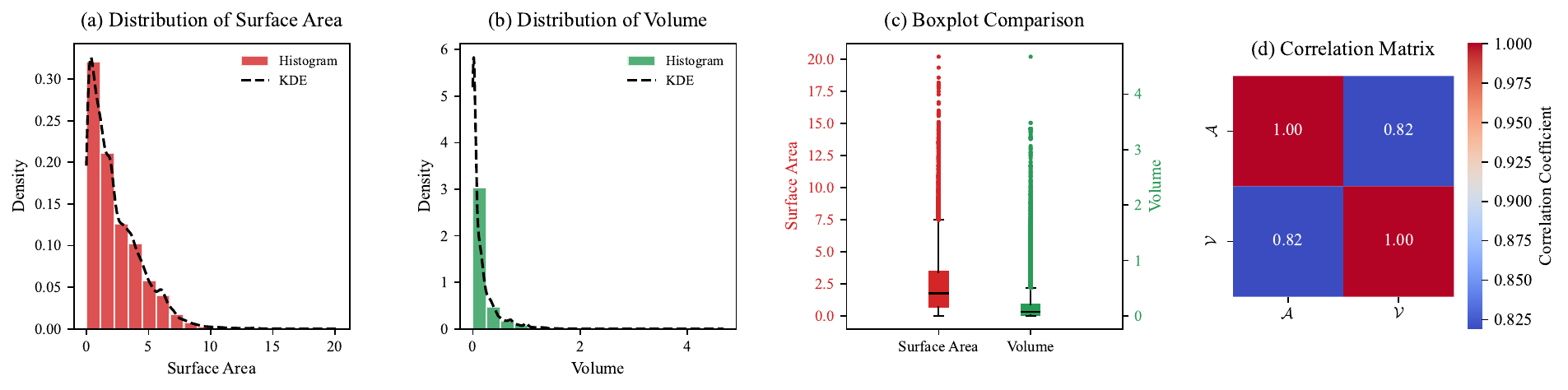}
    \vspace{-0.8cm}
    \caption{Statistical analysis of geometric properties in the proposed CAD dataset: (a) Distribution of surface area; (b) Distribution of volume; (c) Dual-axis boxplots comparing area and volume; (d) Correlation heatmap illustrating covariance between geometric attributes.}
    \vspace{-12pt}
    \label{fig:data}	
\end{figure*}

\begin{table*}[htbp]
\centering
\caption{\centering Comprehensive evaluation of condition-guided CAD generation under multi-condition supervision, where both surface area and volume are jointly used as constraints. Bold: best, \underline{underline}: second best.}
\vspace{-0.3cm}
\begin{tabular}{c||ccc|ccc}
\toprule
{\textbf{}} & \multicolumn{3}{c|}{\bf{Surface Area Constraints}} & \multicolumn{3}{c}{\bf{Volume Constraints}} \\
\cmidrule(lr){2-4}
\cmidrule(lr){5-7}
{\bf{Method}} & {MSE $\downarrow$~} & {MAE $\downarrow$~} & {PCC $\uparrow$~} & {MSE $\downarrow$~} & {MAE $\downarrow$~} & {PCC $\uparrow$~} \\
\midrule
LSTM~\cite{sundermeyer2012lstm}              & 16.328 & 3.2370 & 0.6101 & 1.0003 & 0.7189 & 0.3977 \\
Transformer (Non-Autoregressive)~\cite{gu2018non} & 11.083 & 2.3903 & 0.5592 & 0.4057 & 0.4253 & 0.5254 \\
Transformer (Autoregressive)~\cite{vaswani2017attention}  & 1.3154 & \underline{0.4098} & 0.8338 & 0.0621 & 0.0629 & 0.6399 \\
DeepCAD~\cite{wu2021deepcad}  & 12.708 & 2.4851 & 0.6091 & 0.6371 & 0.4642 & 0.5477 \\
D3PM~\cite{austin2021structured}  & 8.0044 & 0.5789 & 0.7762 & 0.2080 & 0.0655 & 0.8062 \\
BFN~\cite{graves2023bayesian}  & \underline{1.1745} & 0.4935 & \underline{0.8956} & \underline{0.0340} & \underline{0.0577} & \underline{0.8986} \\ 
\midrule
\rowcolor{gray!15}
Ours     & \textbf{0.4235} & \textbf{0.3128} & \textbf{0.9512} & \textbf{0.0078} & \textbf{0.0326} & \textbf{0.9652} \\
\bottomrule
\end{tabular}
\label{tab:va}
\end{table*}

To address the absence of quantitative descriptors in existing CAD datasets and validate our conditional CAD generative method, we construct the first quantitatively constrained CAD dataset, in which each CAD model is paired with quantitative conditions essential for design. Drawing on the DeepCAD dataset~\cite{wu2021deepcad}, which comprises 178,238 parametric models, we compute surface area $\mathcal{A}$ and volume $\mathcal{V}$ through numerical integration by performing boundary representation (B-Rep) analysis with PythonOCC~\cite{paviot_pythonocc_2022}. To ensure data integrity, we adopt a duplicate detection and removal strategy consistent with previous studies~\cite{willis2021engineering,xu2022skexgen,xu2023hierarchical} and retain only those CAD models that contain no more than 64 parametric CAD tokens. Following duplicate removal and filtering, the final dataset comprises 68,219 training instances, 6,327 validation instances, and 5,652 testing instances, with each instance pairing a parametric CAD sequence with its corresponding geometric properties to enable direct supervision for condition-aware parametric CAD sequence generation.

Figure~\ref{fig:data} summarizes key statistics of the computed surface area and volume values. Both distributions (subfigures (a) and (b)) are right-skewed, indicating that most models have small areas and volumes. The dual-axis boxplot in subfigure (c) reveals that volume values have a wider interquartile range, while subfigure (d) shows a strong positive correlation (Pearson’s \(r = 0.82\)) between the two properties, confirming geometric consistency and label reliability.

\subsubsection{Baseline Methods and Metrics}

Existing methods do not comprehensively address the task of quantitatively constrained CAD generation. To establish effective baselines, we adapt three representative generative paradigms: (1) sequential architectures, including LSTM models~\cite{sundermeyer2012lstm} with condition-aware decoding; (2) Transformer-based models, encompassing both autoregressive~\cite{vaswani2017attention} and non-autoregressive~\cite{gu2018non} variants augmented with condition tokens; and (3) specialized generative models, such as DeepCAD~\cite{wu2021deepcad}, which incorporates geometric conditioning, and D3PM~\cite{austin2021structured}, which employs guided diffusion. These adapted baselines provide a comprehensive benchmark for evaluating condition-aware parametric CAD generation under quantitative geometric constraints.

To evaluate model performance in quantitatively constrained CAD generation, we adopt three metrics: (1) \textit{Mean Squared Error} (MSE) to capture deviations critical to manufacturing precision; (2) \textit{Mean Absolute Error} (MAE) to measure average deviation magnitude; and (3) \textit{Pearson Correlation Coefficient} (PCC) to assess linear alignment with target properties. Additionally, we define three evaluation scenarios: two single-condition tasks and one multi-condition task, enabling a thorough assessment of model performance under varying levels of constraint complexity.

\begin{table}[htbp]
\centering
\caption{Comparison of condition-guided CAD generation under single-property supervision. Models are evaluated independently under surface area and volume constraints. Bold: best, \underline{underline}: second best.}
\vspace{-0.3cm}
\resizebox{1.0\linewidth}{!}{
\begin{tabular}{c||ccc}
\toprule
{\bf{Method}} & {MSE $\downarrow$~} & {MAE $\downarrow$~} & {PCC $\uparrow$~} \\
\toprule
\multicolumn{4}{c}{\textbf{Surface Area Constraints}} \\
\toprule
LSTM~\cite{sundermeyer2012lstm}    & 20.623 & 3.3107 & 0.7209 \\
Transformer (Non-Autoregressive)~\cite{gu2018non} & 30.807 & 2.1219 & 0.5582 \\
Transformer (Autoregressive)~\cite{vaswani2017attention} & \underline{1.3033} & \underline{0.7678} & \underline{0.8410} \\
DeepCAD~\cite{wu2021deepcad}   & 8.8662 & 1.8997 & 0.6120 \\
D3PM~\cite{austin2021structured}  & 45.194 & 0.9814 & 0.6790 \\
BFN~\cite{graves2023bayesian}  & 5.1052 & 0.7854 & 0.7865 \\
\midrule
\rowcolor{gray!15}
Ours  & \textbf{0.3503} & \textbf{0.3067} & \textbf{0.9587} \\
\bottomrule
\multicolumn{4}{c}{\textbf{Volume Constraints}} \\
\toprule
LSTM~\cite{sundermeyer2012lstm}   & 0.2470 & 0.4134 & 0.7513 \\
Transformer (Non-Autoregressive)~\cite{gu2018non} & 0.4160 & 0.4308 & 0.6085 \\
Transformer (Autoregressive)~\cite{vaswani2017attention}  & \underline{0.0685} & 0.1097 & 0.6645 \\
DeepCAD~\cite{wu2021deepcad}  & 0.4301 & 0.3707 & 0.4891 \\
D3PM~\cite{austin2021structured}  & 1.5751 & \underline{0.0861} & 0.7703 \\
BFN~\cite{graves2023bayesian} & 0.2188 & 0.1030 & \underline{0.7748} \\
\midrule
\rowcolor{gray!15}
Ours & \textbf{0.0158} & \textbf{0.0529} & \textbf{0.9342} \\ 
\bottomrule
\end{tabular}
}
\label{tab:area}
\end{table}

\begin{figure}[htbp]
    \centering
    \includegraphics[width=1.0\linewidth]{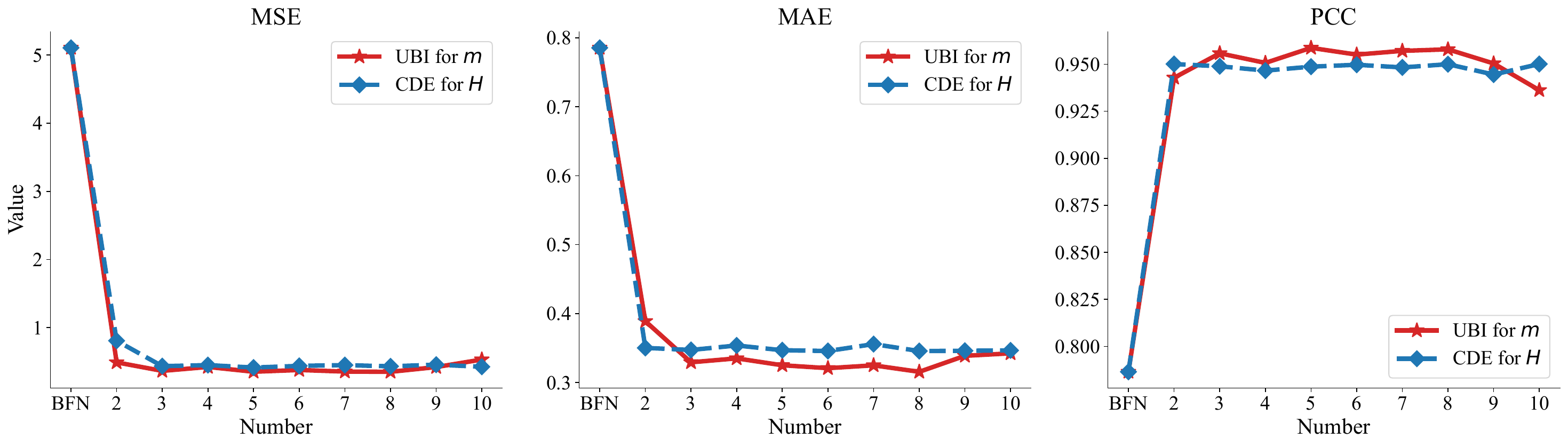}
    \vspace{-7pt}
    \caption{
    Ablation study on the number of parallel samples $m$ in our Unbiased Bayesian Inference (UBI) and sampling granularity $H$ in our Calibrated Distribution Estimation (CDE). Performance consistently improves across all metrics as $m$ and $H$ increase.
    }
    \vspace{-0.3cm}
    \label{fig:ablation_m}
\end{figure}

\subsection{Main Results}

\subsubsection{Performance under Single-Condition Constraints}
\label{subsubsec:result_single}

We begin with the single-condition setting, where the model is guided by either surface area or volume individually. As summarized in Table~\ref{tab:area}, our method consistently outperforms all baselines across key evaluation metrics. Compared to strong alternatives such as BFN and the autoregressive Transformer, our approach reduces mean squared error (MSE) by over 60\% and improves Pearson correlation scores by 7–20\%, reflecting more accurate numerical alignment and greater robustness under geometric constraints. These improvements highlight the model's capacity to internalize and respond to scalar supervision, even in the absence of additional auxiliary signals.

Conventional baselines struggle to balance shape plausibility and constraint satisfaction: DeepCAD and LSTM exhibit poor generalization, while D3PM and non-autoregressive Transformers underperform significantly in accuracy. These results validate the effectiveness of our target-aware guidance and unbiased inference mechanism in modeling property-conditioned CAD generation, even when supervision is limited to a single geometric attribute.

\subsubsection{Performance under Multi-Condition Constraints}
\label{subsubsec:result_multi}

We further evaluate the models under multi-condition supervision, where surface area and volume are jointly imposed as constraints. As shown in Table~\ref{tab:va}, our method maintains its superiority, achieving over 2× improvement in MSE compared to the best-performing baseline. Meanwhile, correlation scores increase by more than 5\%, demonstrating the model’s ability to capture cross-property dependencies while maintaining numerical fidelity.

In contrast, competitive baselines such as BFN and D3PM exhibit relatively strong correlations but suffer from high absolute errors, revealing a mismatch between structural trends and fine-grained accuracy. Transformer-based methods degrade further in this compound setting. These results highlight the scalability and robustness of our framework in handling high-dimensional constraint spaces, underscoring its practical utility in real-world, multi-objective CAD design scenarios.

\subsubsection{Computational Resource}

Table~\ref{tab:time_memory} reports the effect of varying the number of parallel samples $m$ in our Unbiased Bayesian Inference (UBI) module. The results show that increasing $m$ does not increase runtime but only raises memory usage, meaning substantial accuracy improvements can be achieved without compromising sampling efficiency. For example, setting $m=4$ results in a 75.6\% reduction in MSE with only a 5.7\% increase in memory usage and a slight decrease in runtime. Notably, increasing $m$ beyond this point yields diminishing returns and may even reduce accuracy, suggesting that moderate parallelism provides an effective balance between memory consumption and generative fidelity.

\begin{table}[htbp]
\centering
\caption{Comparison of computational cost and generation accuracy under surface area constraints with varying values of the parallel sampling count \( m \). Results are measured on a single batch. Relative changes are reported with respect to the baseline setting \( m{=}1 \).}
\vspace{-0.3cm}
\resizebox{1.0\linewidth}{!}{%
\begin{tabular}{lccc}
\toprule
Setting  & Memory / Growth  & Runtime / Growth  & MSE / Gain   \\
\midrule
{$m$=1}  & 1351~MB /~ 0.0~\%  &12.06~s /~ +0.0~\% & 0.8469~/~ 0.0~\% \\
{$m$=2}  & 1390~MB /~ 2.9~\%  &12.01~s /~ -0.5~\% &  0.5008~/~ 40.9~\% \\
{$m$=4}  & 1428~MB /~5.7~\%  &11.85~s /~ -1.7~\% &  0.2063~/~ 75.6~\%\\
{$m$=8}  & 1511~MB /~ 11.8~\% &11.97~s   /~ -0.7~\% & 0.4850~/~ 42.7~\% \\
{$m$=16}  & 1977~MB /~ 46.3~\% &12.20~s   /~ +1.2~\% & 0.7191~/~ 15.1~\% \\
\bottomrule
\end{tabular}}
\label{tab:time_memory}
\end{table}

\subsection{Ablation Study} \label{subsec:ablation}

\subsubsection{Impact of Parallel Samples $m$ and Calibration Granularity $H$}

We investigate the influence of two critical sampling parameters in our framework: the number of parallel samples $m$ in Unbiased Bayesian Inference (UBI) and the sampling granularity $H$ in Calibrated Distribution Estimation (CDE), as illustrated in Figure~\ref{fig:ablation_m}. Both parameters play essential roles in enhancing generative performance. Increasing $m$ improves MSE, MAE, and PCC by stabilizing posterior estimates and mitigating exposure bias; however, performance gains saturate beyond a modest threshold, suggesting that small to moderate values offer a practical trade-off between accuracy and computational cost. Similarly, higher $H$ values lead to more precise categorical approximations in CDE, yet diminishing returns are observed as granularity increases. These findings indicate that moderate settings for both $m$ and $H$ are sufficient to realize the majority of performance benefits. Taken together, UBI and CDE offer complementary strengths—posterior stabilization and distributional refinement—that jointly improve sample quality through principled bias correction.

\begin{figure*}[htbp]
    \centering
    \includegraphics[width=0.9\linewidth]{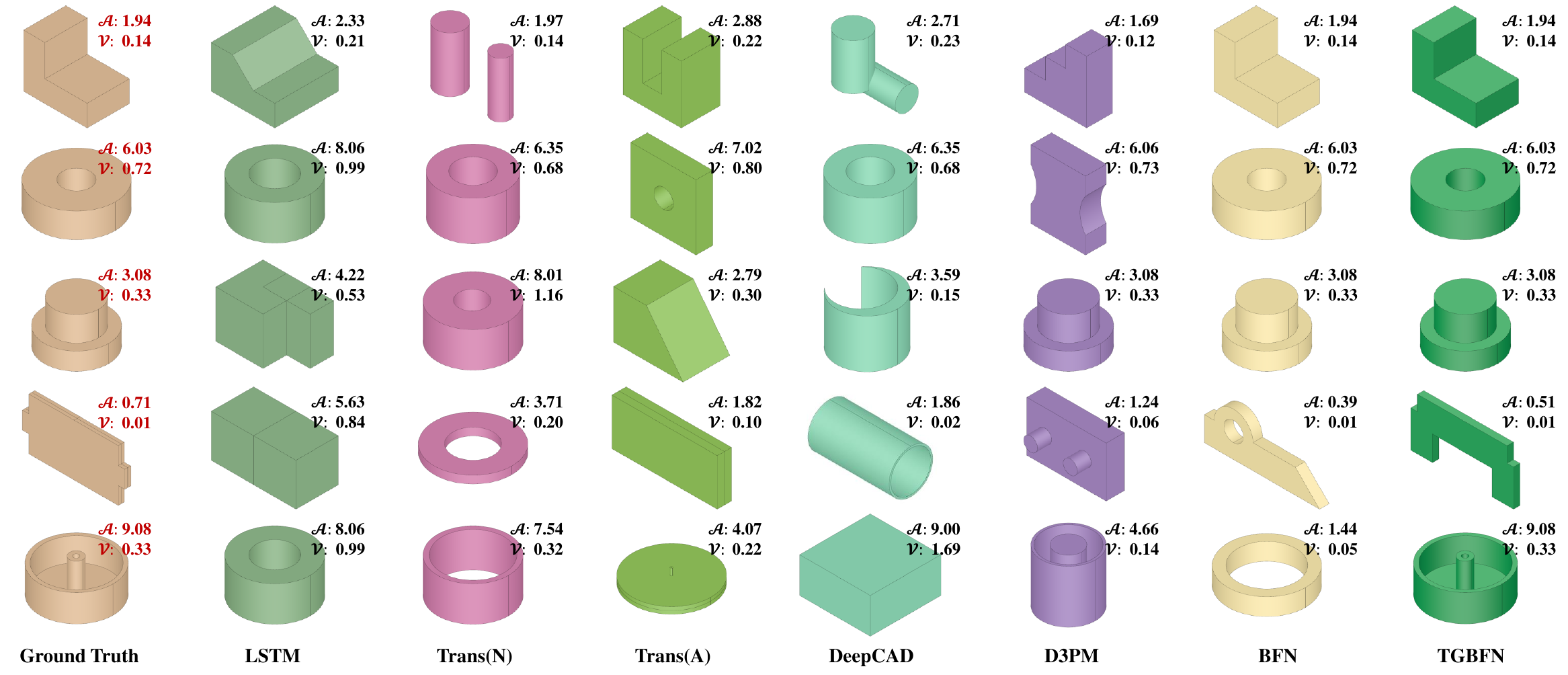}
    \vspace{-0.3cm}
    \caption{Case study of CAD sequence generation under surface area and volume constraints. Each row corresponds to a unique area-volume target pair. The leftmost column depicts the ground truth CAD sequence visualizations, while the remaining columns present outputs generated by various models under the specified constraints.} 
    \vspace{-6pt}
    \label{fig:case}	
\end{figure*}

\begin{table*}[htbp]
\centering
\caption{\centering Ablation study evaluating the contributions of key components in our framework under multi-condition supervision, where surface area and volume are jointly imposed as constraints.}
\vspace{-0.2cm}
\begin{tabular}{l||ccc|ccc}
\toprule
{} & \multicolumn{3}{c|}{\bf{Surface Area Constraints}} & 
\multicolumn{3}{c}{\textbf{Volume Constraints}} \\
\cmidrule(lr){2-4} 
\cmidrule(lr){5-7}
\bf{Model} & {MSE $\downarrow$} & {MAE $\downarrow$} & {PCC $\uparrow$} & {MSE $\downarrow$} & {MAE $\downarrow$} & {PCC $\uparrow$} \\
\midrule
Base Model        & 1.1745 & 0.4935 & 0.8956 & 0.0340 & 0.0577 & 0.8986 \\
Base + UBI    & 0.5108 & 0.3368 & 0.9454 & 0.0100 & 0.0376 & 0.9508 \\
Base + UBI \& GBF  & 0.4302 & 0.3350 & 0.9494 & 0.0109 & 0.0381 & 0.9469 \\
\midrule
\rowcolor{gray!15}
Base + UBI \& GBF \& CDE~(Ours) & \textbf{0.4235} & \textbf{0.3128} & \textbf{0.9512} & \textbf{0.0078} & \textbf{0.0326} & \textbf{0.9652} \\
\bottomrule
\end{tabular}
\vspace{-0.3cm}
\label{tab:ablation_va}
\end{table*}

\subsubsection{Effect of Individual Framework Modules}
To evaluate the contribution of each component in our framework, we conducted a stepwise ablation study under multi-condition constraints (see Table~\ref{tab:ablation_va}). We started with the base model and gradually added the Unbiased Bayesian Inference (UBI), Guided Bayesian Flow (GBF), and Calibrated Distribution Estimation (CDE) module. The complete model, which incorporates all these modules, demonstrated the best performance under both constraints. This result confirms the complementary nature of each component and highlights the overall effectiveness of the framework.

Specifically, UBI provides notable improvements across all evaluation metrics, particularly by reducing MSE and MAE, emphasizing the importance of mitigating exposure bias during sampling. GBF further improves alignment with target conditions through condition-aware guidance within the update process. CDE contributes additional gains, especially under volume constraints, highlighting its role in improved discretization fidelity.

\subsection{Case Study}

To qualitatively assess constraint satisfaction and geometric plausibility, we conduct a case study by randomly selecting five target surface area $\mathcal{A}$ and volume $\mathcal{V}$ values and visualized the corresponding CAD sequences generated by each method (Figure~\ref{fig:case}). The first column presents the ground-truth CAD models that correspond to the given constraints. Each subsequent column shows the CAD sequences generated by different models under the same constraint, along with the corresponding surface area and volume values computed from the generated geometry.

Among the baselines, LSTM, DeepCAD, and D3PM exhibit noticeable deviations in both surface area and volume, frequently failing to maintain geometric consistency. The non-autoregressive Transformer produces plausible shapes, but its performance varies considerably across samples. Although the autoregressive Transformer and BFN demonstrate relatively better alignment with the target constraints, they remain susceptible to overfitting or drift in challenging cases. In contrast, our method (TGBFN) generates geometries that not only precisely match the target values but also display high structural plausibility, underscoring its effectiveness in constraint-aware CAD generation.

\section{Conclusion}

In this paper, we introduce Target-Guided Bayesian Flow Network (TGBFN), a Bayesian generative framework for quantitatively constrained CAD sequence generation. By incorporating the unbiased inference mechanism, guided Bayesian flow, and calibrated distribution estimation, TGBFN effectively addresses critical challenges such as exposure bias and precise conditional control. To validate our approach, we constructed the first benchmark CAD dataset with explicit quantitative constraints, using surface area and volume as target conditions, and conducted comprehensive experiments across multiple tasks. The results demonstrate that TGBFN consistently produces accurate and condition-aware CAD sequences. Future work will explore more complex constraint formulations and extend the framework to broader CAD applications.

\begin{acks}
We sincerely thank the anonymous reviewers and area chairs for their valuable time and insightful comments. This work was supported by the National Major Scientific Instruments and Equipments Development Project of the National Natural Science Foundation of China (No. 62427820), the Fundamental Research Funds for the Central Universities (No. 1082204112364), the Science Fund for Creative Research Groups of the Sichuan Province Natural Science Foundation (No. 2024NSFTD0035), the Science and Technology Major Project of Sichuan Province (No. 2024ZDZX0003), the National Key R\&D Program of China (No. 2024YFB3312503), and the Natural Science Foundation of Sichuan Province (No. 2024NSFTD0048). We also acknowledge the support of the Sichuan Province Engineering Technology Research Center of Broadband Electronics Intelligent Manufacturing.
\end{acks}

\bibliographystyle{ACM-Reference-Format}
\balance
\bibliography{sample-base}


\begin{thebibliography}{61}


\ifx \showCODEN    \undefined \def \showCODEN     #1{\unskip}     \fi
\ifx \showISBNx    \undefined \def \showISBNx     #1{\unskip}     \fi
\ifx \showISBNxiii \undefined \def \showISBNxiii  #1{\unskip}     \fi
\ifx \showISSN     \undefined \def \showISSN      #1{\unskip}     \fi
\ifx \showLCCN     \undefined \def \showLCCN      #1{\unskip}     \fi
\ifx \shownote     \undefined \def \shownote      #1{#1}          \fi
\ifx \showarticletitle \undefined \def \showarticletitle #1{#1}   \fi
\ifx \showURL      \undefined \def \showURL       {\relax}        \fi
\providecommand\bibfield[2]{#2}
\providecommand\bibinfo[2]{#2}
\providecommand\natexlab[1]{#1}
\providecommand\showeprint[2][]{arXiv:#2}

\bibitem[Atkinson et~al\mbox{.}(2025)]%
        {atkinson2025protein}
\bibfield{author}{\bibinfo{person}{Timothy Atkinson}, \bibinfo{person}{Thomas~D Barrett}, \bibinfo{person}{Scott Cameron}, \bibinfo{person}{Bora Guloglu}, \bibinfo{person}{Matthew Greenig}, \bibinfo{person}{Charlie~B Tan}, \bibinfo{person}{Louis Robinson}, \bibinfo{person}{Alex Graves}, \bibinfo{person}{Liviu Copoiu}, {and} \bibinfo{person}{Alexandre Laterre}.} \bibinfo{year}{2025}\natexlab{}.
\newblock \showarticletitle{Protein sequence modelling with Bayesian flow networks}.
\newblock \bibinfo{journal}{\emph{Nature Communications}} \bibinfo{volume}{16}, \bibinfo{number}{1} (\bibinfo{year}{2025}), \bibinfo{pages}{3197}.
\newblock


\bibitem[Austin et~al\mbox{.}(2021)]%
        {austin2021structured}
\bibfield{author}{\bibinfo{person}{Jacob Austin}, \bibinfo{person}{Daniel~D Johnson}, \bibinfo{person}{Jonathan Ho}, \bibinfo{person}{Daniel Tarlow}, {and} \bibinfo{person}{Rianne Van Den~Berg}.} \bibinfo{year}{2021}\natexlab{}.
\newblock \showarticletitle{Structured denoising diffusion models in discrete state-spaces}.
\newblock \bibinfo{journal}{\emph{Advances in neural information processing systems}}  \bibinfo{volume}{34} (\bibinfo{year}{2021}), \bibinfo{pages}{17981--17993}.
\newblock


\bibitem[Bai et~al\mbox{.}(2024)]%
        {bai2024regularized}
\bibfield{author}{\bibinfo{person}{Yunpeng Bai}, \bibinfo{person}{Ruisheng Wang}, {and} \bibinfo{person}{Qixing Huang}.} \bibinfo{year}{2024}\natexlab{}.
\newblock \bibinfo{title}{Regularized Diffusion Modeling for {CAD} Representation Generation}.
\newblock
\urldef\tempurl%
\url{https://openreview.net/forum?id=sCGIbhv4Yv}
\showURL{%
\tempurl}


\bibitem[Camba et~al\mbox{.}(2016)]%
        {camba2016parametric}
\bibfield{author}{\bibinfo{person}{Jorge~D Camba}, \bibinfo{person}{Manuel Contero}, {and} \bibinfo{person}{Pedro Company}.} \bibinfo{year}{2016}\natexlab{}.
\newblock \showarticletitle{Parametric CAD modeling: An analysis of strategies for design reusability}.
\newblock \bibinfo{journal}{\emph{Computer-aided design}}  \bibinfo{volume}{74} (\bibinfo{year}{2016}), \bibinfo{pages}{18--31}.
\newblock


\bibitem[Chen et~al\mbox{.}(2024)]%
        {chen2024img2cad}
\bibfield{author}{\bibinfo{person}{Tianrun Chen}, \bibinfo{person}{Chunan Yu}, \bibinfo{person}{Yuanqi Hu}, \bibinfo{person}{Jing Li}, \bibinfo{person}{Tao Xu}, \bibinfo{person}{Runlong Cao}, \bibinfo{person}{Lanyun Zhu}, \bibinfo{person}{Ying Zang}, \bibinfo{person}{Yong Zhang}, \bibinfo{person}{Zejian Li}, {et~al\mbox{.}}} \bibinfo{year}{2024}\natexlab{}.
\newblock \showarticletitle{Img2cad: Conditioned 3d cad model generation from single image with structured visual geometry}.
\newblock \bibinfo{journal}{\emph{arXiv preprint arXiv:2410.03417}} (\bibinfo{year}{2024}).
\newblock


\bibitem[Dhariwal and Nichol(2021)]%
        {dhariwal2021diffusion}
\bibfield{author}{\bibinfo{person}{Prafulla Dhariwal} {and} \bibinfo{person}{Alexander Nichol}.} \bibinfo{year}{2021}\natexlab{}.
\newblock \showarticletitle{Diffusion models beat gans on image synthesis}.
\newblock \bibinfo{journal}{\emph{Advances in neural information processing systems}}  \bibinfo{volume}{34} (\bibinfo{year}{2021}), \bibinfo{pages}{8780--8794}.
\newblock


\bibitem[Fan et~al\mbox{.}(2022)]%
        {fan2022cadtransformer}
\bibfield{author}{\bibinfo{person}{Zhiwen Fan}, \bibinfo{person}{Tianlong Chen}, \bibinfo{person}{Peihao Wang}, {and} \bibinfo{person}{Zhangyang Wang}.} \bibinfo{year}{2022}\natexlab{}.
\newblock \showarticletitle{Cadtransformer: Panoptic symbol spotting transformer for cad drawings}. In \bibinfo{booktitle}{\emph{Proceedings of the IEEE/CVF Conference on Computer Vision and Pattern Recognition}}. \bibinfo{pages}{10986--10996}.
\newblock


\bibitem[Ganin et~al\mbox{.}(2021)]%
        {ganin2021computer}
\bibfield{author}{\bibinfo{person}{Yaroslav Ganin}, \bibinfo{person}{Sergey Bartunov}, \bibinfo{person}{Yujia Li}, \bibinfo{person}{Ethan Keller}, {and} \bibinfo{person}{Stefano Saliceti}.} \bibinfo{year}{2021}\natexlab{}.
\newblock \showarticletitle{Computer-aided design as language}.
\newblock \bibinfo{journal}{\emph{Advances in Neural Information Processing Systems}}  \bibinfo{volume}{34} (\bibinfo{year}{2021}), \bibinfo{pages}{5885--5897}.
\newblock


\bibitem[Gao et~al\mbox{.}(2024)]%
        {gao2024diffcad}
\bibfield{author}{\bibinfo{person}{Daoyi Gao}, \bibinfo{person}{D{\'a}vid Rozenberszki}, \bibinfo{person}{Stefan Leutenegger}, {and} \bibinfo{person}{Angela Dai}.} \bibinfo{year}{2024}\natexlab{}.
\newblock \showarticletitle{Diffcad: Weakly-supervised probabilistic cad model retrieval and alignment from an rgb image}.
\newblock \bibinfo{journal}{\emph{ACM Transactions on Graphics (TOG)}} \bibinfo{volume}{43}, \bibinfo{number}{4} (\bibinfo{year}{2024}), \bibinfo{pages}{1--15}.
\newblock


\bibitem[Graves et~al\mbox{.}(2023)]%
        {graves2023bayesian}
\bibfield{author}{\bibinfo{person}{Alex Graves}, \bibinfo{person}{Rupesh~Kumar Srivastava}, \bibinfo{person}{Timothy Atkinson}, {and} \bibinfo{person}{Faustino Gomez}.} \bibinfo{year}{2023}\natexlab{}.
\newblock \showarticletitle{Bayesian flow networks}.
\newblock \bibinfo{journal}{\emph{arXiv preprint arXiv:2308.07037}} (\bibinfo{year}{2023}).
\newblock


\bibitem[Gu et~al\mbox{.}(2018)]%
        {gu2018non}
\bibfield{author}{\bibinfo{person}{J Gu}, \bibinfo{person}{J Bradbury}, \bibinfo{person}{C Xiong}, \bibinfo{person}{VOK Li}, {and} \bibinfo{person}{R Socher}.} \bibinfo{year}{2018}\natexlab{}.
\newblock \showarticletitle{Non-autoregressive neural machine translation}. In \bibinfo{booktitle}{\emph{International Conference on Learning Representations (ICLR)}}.
\newblock


\bibitem[Guo et~al\mbox{.}(2022)]%
        {guo2022complexgen}
\bibfield{author}{\bibinfo{person}{Haoxiang Guo}, \bibinfo{person}{Shilin Liu}, \bibinfo{person}{Hao Pan}, \bibinfo{person}{Yang Liu}, \bibinfo{person}{Xin Tong}, {and} \bibinfo{person}{Baining Guo}.} \bibinfo{year}{2022}\natexlab{}.
\newblock \showarticletitle{Complexgen: Cad reconstruction by b-rep chain complex generation}.
\newblock \bibinfo{journal}{\emph{ACM Transactions on Graphics (TOG)}} \bibinfo{volume}{41}, \bibinfo{number}{4} (\bibinfo{year}{2022}), \bibinfo{pages}{1--18}.
\newblock


\bibitem[Ho et~al\mbox{.}(2020)]%
        {ho2020denoising}
\bibfield{author}{\bibinfo{person}{Jonathan Ho}, \bibinfo{person}{Ajay Jain}, {and} \bibinfo{person}{Pieter Abbeel}.} \bibinfo{year}{2020}\natexlab{}.
\newblock \showarticletitle{Denoising diffusion probabilistic models}.
\newblock \bibinfo{journal}{\emph{Advances in neural information processing systems}}  \bibinfo{volume}{33} (\bibinfo{year}{2020}), \bibinfo{pages}{6840--6851}.
\newblock


\bibitem[Hoffmann and Kim(2001)]%
        {hoffmann2001towards}
\bibfield{author}{\bibinfo{person}{Christoph~M Hoffmann} {and} \bibinfo{person}{K-J Kim}.} \bibinfo{year}{2001}\natexlab{}.
\newblock \showarticletitle{Towards valid parametric CAD models}.
\newblock \bibinfo{journal}{\emph{Computer-Aided Design}} \bibinfo{volume}{33}, \bibinfo{number}{1} (\bibinfo{year}{2001}), \bibinfo{pages}{81--90}.
\newblock


\bibitem[Jayaraman et~al\mbox{.}(2022)]%
        {jayaraman2022solidgen}
\bibfield{author}{\bibinfo{person}{Pradeep~Kumar Jayaraman}, \bibinfo{person}{Joseph~G Lambourne}, \bibinfo{person}{Nishkrit Desai}, \bibinfo{person}{Karl~DD Willis}, \bibinfo{person}{Aditya Sanghi}, {and} \bibinfo{person}{Nigel~JW Morris}.} \bibinfo{year}{2022}\natexlab{}.
\newblock \showarticletitle{Solidgen: An autoregressive model for direct b-rep synthesis}.
\newblock \bibinfo{journal}{\emph{arXiv preprint arXiv:2203.13944}} (\bibinfo{year}{2022}).
\newblock


\bibitem[Jayaraman et~al\mbox{.}(2021)]%
        {jayaraman2021uv}
\bibfield{author}{\bibinfo{person}{Pradeep~Kumar Jayaraman}, \bibinfo{person}{Aditya Sanghi}, \bibinfo{person}{Joseph~G Lambourne}, \bibinfo{person}{Karl~DD Willis}, \bibinfo{person}{Thomas Davies}, \bibinfo{person}{Hooman Shayani}, {and} \bibinfo{person}{Nigel Morris}.} \bibinfo{year}{2021}\natexlab{}.
\newblock \showarticletitle{Uv-net: Learning from boundary representations}. In \bibinfo{booktitle}{\emph{Proceedings of the IEEE/CVF conference on computer vision and pattern recognition}}. \bibinfo{pages}{11703--11712}.
\newblock


\bibitem[Jones et~al\mbox{.}(2021)]%
        {jones2021automate}
\bibfield{author}{\bibinfo{person}{Benjamin Jones}, \bibinfo{person}{Dalton Hildreth}, \bibinfo{person}{Duowen Chen}, \bibinfo{person}{Ilya Baran}, \bibinfo{person}{Vladimir~G Kim}, {and} \bibinfo{person}{Adriana Schulz}.} \bibinfo{year}{2021}\natexlab{}.
\newblock \showarticletitle{Automate: A dataset and learning approach for automatic mating of cad assemblies}.
\newblock \bibinfo{journal}{\emph{ACM Transactions on Graphics (TOG)}} \bibinfo{volume}{40}, \bibinfo{number}{6} (\bibinfo{year}{2021}), \bibinfo{pages}{1--18}.
\newblock


\bibitem[Jones et~al\mbox{.}(2023)]%
        {jones2023self}
\bibfield{author}{\bibinfo{person}{Benjamin~T Jones}, \bibinfo{person}{Michael Hu}, \bibinfo{person}{Milin Kodnongbua}, \bibinfo{person}{Vladimir~G Kim}, {and} \bibinfo{person}{Adriana Schulz}.} \bibinfo{year}{2023}\natexlab{}.
\newblock \showarticletitle{Self-supervised representation learning for cad}. In \bibinfo{booktitle}{\emph{Proceedings of the IEEE/CVF Conference on Computer Vision and Pattern Recognition}}. \bibinfo{pages}{21327--21336}.
\newblock


\bibitem[Khan et~al\mbox{.}(2024a)]%
        {khan2024cad}
\bibfield{author}{\bibinfo{person}{Mohammad~Sadil Khan}, \bibinfo{person}{Elona Dupont}, \bibinfo{person}{Sk~Aziz Ali}, \bibinfo{person}{Kseniya Cherenkova}, \bibinfo{person}{Anis Kacem}, {and} \bibinfo{person}{Djamila Aouada}.} \bibinfo{year}{2024}\natexlab{a}.
\newblock \showarticletitle{Cad-signet: Cad language inference from point clouds using layer-wise sketch instance guided attention}. In \bibinfo{booktitle}{\emph{Proceedings of the IEEE/CVF Conference on Computer Vision and Pattern Recognition}}. \bibinfo{pages}{4713--4722}.
\newblock


\bibitem[Khan et~al\mbox{.}(2024b)]%
        {khan2024text2cad}
\bibfield{author}{\bibinfo{person}{Mohammad~Sadil Khan}, \bibinfo{person}{Sankalp Sinha}, \bibinfo{person}{Talha Uddin}, \bibinfo{person}{Didier Stricker}, \bibinfo{person}{Sk~Aziz Ali}, {and} \bibinfo{person}{Muhammad~Zeshan Afzal}.} \bibinfo{year}{2024}\natexlab{b}.
\newblock \showarticletitle{Text2cad: Generating sequential cad designs from beginner-to-expert level text prompts}.
\newblock \bibinfo{journal}{\emph{Advances in Neural Information Processing Systems}}  \bibinfo{volume}{37} (\bibinfo{year}{2024}), \bibinfo{pages}{7552--7579}.
\newblock


\bibitem[Koch et~al\mbox{.}(2019)]%
        {koch2019abc}
\bibfield{author}{\bibinfo{person}{Sebastian Koch}, \bibinfo{person}{Albert Matveev}, \bibinfo{person}{Zhongshi Jiang}, \bibinfo{person}{Francis Williams}, \bibinfo{person}{Alexey Artemov}, \bibinfo{person}{Evgeny Burnaev}, \bibinfo{person}{Marc Alexa}, \bibinfo{person}{Denis Zorin}, {and} \bibinfo{person}{Daniele Panozzo}.} \bibinfo{year}{2019}\natexlab{}.
\newblock \showarticletitle{Abc: A big cad model dataset for geometric deep learning}. In \bibinfo{booktitle}{\emph{Proceedings of the IEEE/CVF conference on computer vision and pattern recognition}}. \bibinfo{pages}{9601--9611}.
\newblock


\bibitem[Kodnongbua et~al\mbox{.}(2023)]%
        {kodnongbua2023reparamcad}
\bibfield{author}{\bibinfo{person}{Milin Kodnongbua}, \bibinfo{person}{Benjamin Jones}, \bibinfo{person}{Maaz Bin~Safeer Ahmad}, \bibinfo{person}{Vladimir Kim}, {and} \bibinfo{person}{Adriana Schulz}.} \bibinfo{year}{2023}\natexlab{}.
\newblock \showarticletitle{Reparamcad: Zero-shot cad re-parameterization for interactive manipulation}. In \bibinfo{booktitle}{\emph{SIGGRAPH Asia 2023 Conference Papers}}. \bibinfo{pages}{1--12}.
\newblock


\bibitem[Kullback and Leibler(1951)]%
        {kullback1951information}
\bibfield{author}{\bibinfo{person}{Solomon Kullback} {and} \bibinfo{person}{Richard~A Leibler}.} \bibinfo{year}{1951}\natexlab{}.
\newblock \showarticletitle{On information and sufficiency}.
\newblock \bibinfo{journal}{\emph{The annals of mathematical statistics}} \bibinfo{volume}{22}, \bibinfo{number}{1} (\bibinfo{year}{1951}), \bibinfo{pages}{79--86}.
\newblock


\bibitem[Lambourne et~al\mbox{.}(2022)]%
        {lambourne2022reconstructing}
\bibfield{author}{\bibinfo{person}{Joseph~George Lambourne}, \bibinfo{person}{Karl Willis}, \bibinfo{person}{Pradeep~Kumar Jayaraman}, \bibinfo{person}{Longfei Zhang}, \bibinfo{person}{Aditya Sanghi}, {and} \bibinfo{person}{Kamal~Rahimi Malekshan}.} \bibinfo{year}{2022}\natexlab{}.
\newblock \showarticletitle{Reconstructing editable prismatic cad from rounded voxel models}. In \bibinfo{booktitle}{\emph{SIGGRAPH Asia 2022 Conference Papers}}. \bibinfo{pages}{1--9}.
\newblock


\bibitem[Li et~al\mbox{.}(2023)]%
        {li2023secad}
\bibfield{author}{\bibinfo{person}{Pu Li}, \bibinfo{person}{Jianwei Guo}, \bibinfo{person}{Xiaopeng Zhang}, {and} \bibinfo{person}{Dong-Ming Yan}.} \bibinfo{year}{2023}\natexlab{}.
\newblock \showarticletitle{Secad-net: Self-supervised cad reconstruction by learning sketch-extrude operations}. In \bibinfo{booktitle}{\emph{Proceedings of the IEEE/CVF Conference on Computer Vision and Pattern Recognition}}. \bibinfo{pages}{16816--16826}.
\newblock


\bibitem[Li et~al\mbox{.}(2024)]%
        {li2024cad}
\bibfield{author}{\bibinfo{person}{Xueyang Li}, \bibinfo{person}{Yu Song}, \bibinfo{person}{Yunzhong Lou}, {and} \bibinfo{person}{Xiangdong Zhou}.} \bibinfo{year}{2024}\natexlab{}.
\newblock \showarticletitle{CAD Translator: An Effective Drive for Text to 3D Parametric Computer-Aided Design Generative Modeling}. In \bibinfo{booktitle}{\emph{Proceedings of the 32nd ACM International Conference on Multimedia}}. \bibinfo{pages}{8461--8470}.
\newblock


\bibitem[Liu et~al\mbox{.}(2021)]%
        {liu2021simulated}
\bibfield{author}{\bibinfo{person}{Xianggen Liu}, \bibinfo{person}{Pengyong Li}, \bibinfo{person}{Fandong Meng}, \bibinfo{person}{Hao Zhou}, \bibinfo{person}{Huasong Zhong}, \bibinfo{person}{Jie Zhou}, \bibinfo{person}{Lili Mou}, {and} \bibinfo{person}{Sen Song}.} \bibinfo{year}{2021}\natexlab{}.
\newblock \showarticletitle{Simulated annealing for optimization of graphs and sequences}.
\newblock \bibinfo{journal}{\emph{Neurocomputing}}  \bibinfo{volume}{465} (\bibinfo{year}{2021}), \bibinfo{pages}{310--324}.
\newblock


\bibitem[Lou et~al\mbox{.}(2023)]%
        {lou2023brep}
\bibfield{author}{\bibinfo{person}{Yunzhong Lou}, \bibinfo{person}{Xueyang Li}, \bibinfo{person}{Haotian Chen}, {and} \bibinfo{person}{Xiangdong Zhou}.} \bibinfo{year}{2023}\natexlab{}.
\newblock \showarticletitle{Brep-bert: Pre-training boundary representation BERT with sub-graph node contrastive learning}. In \bibinfo{booktitle}{\emph{Proceedings of the 32nd ACM International Conference on Information and Knowledge Management}}. \bibinfo{pages}{1657--1666}.
\newblock


\bibitem[Lv et~al\mbox{.}(2021)]%
        {lv2021voxel}
\bibfield{author}{\bibinfo{person}{Chenlei Lv}, \bibinfo{person}{Weisi Lin}, {and} \bibinfo{person}{Baoquan Zhao}.} \bibinfo{year}{2021}\natexlab{}.
\newblock \showarticletitle{Voxel structure-based mesh reconstruction from a 3D point cloud}.
\newblock \bibinfo{journal}{\emph{IEEE Transactions on Multimedia}}  \bibinfo{volume}{24} (\bibinfo{year}{2021}), \bibinfo{pages}{1815--1829}.
\newblock


\bibitem[Ma et~al\mbox{.}(2024)]%
        {ma2024draw}
\bibfield{author}{\bibinfo{person}{Weijian Ma}, \bibinfo{person}{Shuaiqi Chen}, \bibinfo{person}{Yunzhong Lou}, \bibinfo{person}{Xueyang Li}, {and} \bibinfo{person}{Xiangdong Zhou}.} \bibinfo{year}{2024}\natexlab{}.
\newblock \showarticletitle{Draw Step by Step: Reconstructing CAD Construction Sequences from Point Clouds via Multimodal Diffusion.}. In \bibinfo{booktitle}{\emph{Proceedings of the IEEE/CVF Conference on Computer Vision and Pattern Recognition}}. \bibinfo{pages}{27154--27163}.
\newblock


\bibitem[Niu et~al\mbox{.}(2024)]%
        {niu2024neural}
\bibfield{author}{\bibinfo{person}{Chaoqun Niu}, \bibinfo{person}{Dongdong Chen}, \bibinfo{person}{Jizhe Zhou}, \bibinfo{person}{Jian Wang}, \bibinfo{person}{Xiang Luo}, \bibinfo{person}{Quan-Hui Liu}, \bibinfo{person}{Yuan Li}, {and} \bibinfo{person}{Jiancheng Lv}.} \bibinfo{year}{2024}\natexlab{}.
\newblock \showarticletitle{Neural Boneprint: Person Identification from Bones Using Generative Contrastive Deep Learning}. In \bibinfo{booktitle}{\emph{Proceedings of the 32nd ACM International Conference on Multimedia}}. \bibinfo{pages}{7609--7618}.
\newblock


\bibitem[Para et~al\mbox{.}(2021)]%
        {para2021sketchgen}
\bibfield{author}{\bibinfo{person}{Wamiq Para}, \bibinfo{person}{Shariq Bhat}, \bibinfo{person}{Paul Guerrero}, \bibinfo{person}{Tom Kelly}, \bibinfo{person}{Niloy Mitra}, \bibinfo{person}{Leonidas~J Guibas}, {and} \bibinfo{person}{Peter Wonka}.} \bibinfo{year}{2021}\natexlab{}.
\newblock \showarticletitle{Sketchgen: Generating constrained cad sketches}.
\newblock \bibinfo{journal}{\emph{Advances in Neural Information Processing Systems}}  \bibinfo{volume}{34} (\bibinfo{year}{2021}), \bibinfo{pages}{5077--5088}.
\newblock


\bibitem[Paviot(2022)]%
        {paviot_pythonocc_2022}
\bibfield{author}{\bibinfo{person}{Thomas Paviot}.} \bibinfo{year}{2022}\natexlab{}.
\newblock \bibinfo{title}{"pythonocc"}.
\newblock \bibinfo{howpublished}{Zenodo}.
\newblock
\href{https://doi.org/10.5281/zenodo.7471333}{doi:\nolinkurl{10.5281/zenodo.7471333}}


\bibitem[Ren et~al\mbox{.}(2022)]%
        {ren2022extrudenet}
\bibfield{author}{\bibinfo{person}{Daxuan Ren}, \bibinfo{person}{Jianmin Zheng}, \bibinfo{person}{Jianfei Cai}, \bibinfo{person}{Jiatong Li}, {and} \bibinfo{person}{Junzhe Zhang}.} \bibinfo{year}{2022}\natexlab{}.
\newblock \showarticletitle{Extrudenet: Unsupervised inverse sketch-and-extrude for shape parsing}. In \bibinfo{booktitle}{\emph{European Conference on Computer Vision}}. Springer, \bibinfo{pages}{482--498}.
\newblock


\bibitem[Ritchie et~al\mbox{.}(2023)]%
        {ritchie2023neurosymbolic}
\bibfield{author}{\bibinfo{person}{Daniel Ritchie}, \bibinfo{person}{Paul Guerrero}, \bibinfo{person}{R~Kenny Jones}, \bibinfo{person}{Niloy~J Mitra}, \bibinfo{person}{Adriana Schulz}, \bibinfo{person}{Karl~DD Willis}, {and} \bibinfo{person}{Jiajun Wu}.} \bibinfo{year}{2023}\natexlab{}.
\newblock \showarticletitle{Neurosymbolic models for computer graphics}. In \bibinfo{booktitle}{\emph{Computer graphics forum}}, Vol.~\bibinfo{volume}{42}. Wiley Online Library, \bibinfo{pages}{545--568}.
\newblock


\bibitem[Rombach et~al\mbox{.}(2022)]%
        {rombach2022high}
\bibfield{author}{\bibinfo{person}{Robin Rombach}, \bibinfo{person}{Andreas Blattmann}, \bibinfo{person}{Dominik Lorenz}, \bibinfo{person}{Patrick Esser}, {and} \bibinfo{person}{Bj{\"o}rn Ommer}.} \bibinfo{year}{2022}\natexlab{}.
\newblock \showarticletitle{High-resolution image synthesis with latent diffusion models}. In \bibinfo{booktitle}{\emph{Proceedings of the IEEE/CVF conference on computer vision and pattern recognition}}. \bibinfo{pages}{10684--10695}.
\newblock


\bibitem[Seff et~al\mbox{.}(2021)]%
        {seff2021vitruvion}
\bibfield{author}{\bibinfo{person}{Ari Seff}, \bibinfo{person}{Wenda Zhou}, \bibinfo{person}{Nick Richardson}, {and} \bibinfo{person}{Ryan~P Adams}.} \bibinfo{year}{2021}\natexlab{}.
\newblock \showarticletitle{Vitruvion: A generative model of parametric cad sketches}.
\newblock \bibinfo{journal}{\emph{arXiv preprint arXiv:2109.14124}} (\bibinfo{year}{2021}).
\newblock


\bibitem[Shah(1998)]%
        {shah1998designing}
\bibfield{author}{\bibinfo{person}{Jami~J Shah}.} \bibinfo{year}{1998}\natexlab{}.
\newblock \showarticletitle{Designing with parametric cad: Classification and comparison of construction techniques}. In \bibinfo{booktitle}{\emph{International workshop on geometric modelling}}. Springer, \bibinfo{pages}{53--68}.
\newblock


\bibitem[Sharma et~al\mbox{.}(2020)]%
        {sharma2020parsenet}
\bibfield{author}{\bibinfo{person}{Gopal Sharma}, \bibinfo{person}{Difan Liu}, \bibinfo{person}{Subhransu Maji}, \bibinfo{person}{Evangelos Kalogerakis}, \bibinfo{person}{Siddhartha Chaudhuri}, {and} \bibinfo{person}{Radom{\'\i}r M{\v{e}}ch}.} \bibinfo{year}{2020}\natexlab{}.
\newblock \showarticletitle{Parsenet: A parametric surface fitting network for 3d point clouds}. In \bibinfo{booktitle}{\emph{Computer Vision--ECCV 2020: 16th European Conference, Glasgow, UK, August 23--28, 2020, Proceedings, Part VII 16}}. Springer, \bibinfo{pages}{261--276}.
\newblock


\bibitem[Smirnov et~al\mbox{.}(2019)]%
        {smirnov2019learning}
\bibfield{author}{\bibinfo{person}{Dmitriy Smirnov}, \bibinfo{person}{Mikhail Bessmeltsev}, {and} \bibinfo{person}{Justin Solomon}.} \bibinfo{year}{2019}\natexlab{}.
\newblock \showarticletitle{Learning manifold patch-based representations of man-made shapes}.
\newblock \bibinfo{journal}{\emph{arXiv preprint arXiv:1906.12337}} (\bibinfo{year}{2019}).
\newblock


\bibitem[Song et~al\mbox{.}(2023)]%
        {song2023unified}
\bibfield{author}{\bibinfo{person}{Yuxuan Song}, \bibinfo{person}{Jingjing Gong}, \bibinfo{person}{Hao Zhou}, \bibinfo{person}{Mingyue Zheng}, \bibinfo{person}{Jingjing Liu}, {and} \bibinfo{person}{Wei-Ying Ma}.} \bibinfo{year}{2023}\natexlab{}.
\newblock \showarticletitle{Unified generative modeling of 3d molecules with bayesian flow networks}. In \bibinfo{booktitle}{\emph{The Twelfth International Conference on Learning Representations}}.
\newblock


\bibitem[Sundermeyer et~al\mbox{.}(2012)]%
        {sundermeyer2012lstm}
\bibfield{author}{\bibinfo{person}{Martin Sundermeyer}, \bibinfo{person}{Ralf Schl{\"u}ter}, {and} \bibinfo{person}{Hermann Ney}.} \bibinfo{year}{2012}\natexlab{}.
\newblock \showarticletitle{Lstm neural networks for language modeling.}. In \bibinfo{booktitle}{\emph{Interspeech}}.
\newblock


\bibitem[Tao and Abe(2025)]%
        {tao2025bayesian}
\bibfield{author}{\bibinfo{person}{Nianze Tao} {and} \bibinfo{person}{Minori Abe}.} \bibinfo{year}{2025}\natexlab{}.
\newblock \showarticletitle{Bayesian Flow Network Framework for Chemistry Tasks}.
\newblock \bibinfo{journal}{\emph{Journal of Chemical Information and Modeling}} \bibinfo{volume}{65}, \bibinfo{number}{3} (\bibinfo{year}{2025}), \bibinfo{pages}{1178--1187}.
\newblock


\bibitem[Uy et~al\mbox{.}(2022)]%
        {uy2022point2cyl}
\bibfield{author}{\bibinfo{person}{Mikaela~Angelina Uy}, \bibinfo{person}{Yen-Yu Chang}, \bibinfo{person}{Minhyuk Sung}, \bibinfo{person}{Purvi Goel}, \bibinfo{person}{Joseph~G Lambourne}, \bibinfo{person}{Tolga Birdal}, {and} \bibinfo{person}{Leonidas~J Guibas}.} \bibinfo{year}{2022}\natexlab{}.
\newblock \showarticletitle{Point2cyl: Reverse engineering 3d objects from point clouds to extrusion cylinders}. In \bibinfo{booktitle}{\emph{Proceedings of the IEEE/CVF Conference on Computer Vision and Pattern Recognition}}. \bibinfo{pages}{11850--11860}.
\newblock


\bibitem[Vaswani et~al\mbox{.}(2017)]%
        {vaswani2017attention}
\bibfield{author}{\bibinfo{person}{Ashish Vaswani}, \bibinfo{person}{Noam Shazeer}, \bibinfo{person}{Niki Parmar}, \bibinfo{person}{Jakob Uszkoreit}, \bibinfo{person}{Llion Jones}, \bibinfo{person}{Aidan~N Gomez}, \bibinfo{person}{{\L}ukasz Kaiser}, {and} \bibinfo{person}{Illia Polosukhin}.} \bibinfo{year}{2017}\natexlab{}.
\newblock \showarticletitle{Attention is all you need}.
\newblock \bibinfo{journal}{\emph{Advances in neural information processing systems}}  \bibinfo{volume}{30} (\bibinfo{year}{2017}).
\newblock


\bibitem[Wang et~al\mbox{.}(2022)]%
        {wang2022neural}
\bibfield{author}{\bibinfo{person}{Kehan Wang}, \bibinfo{person}{Jia Zheng}, {and} \bibinfo{person}{Zihan Zhou}.} \bibinfo{year}{2022}\natexlab{}.
\newblock \showarticletitle{Neural face identification in a 2d wireframe projection of a manifold object}. In \bibinfo{booktitle}{\emph{Proceedings of the IEEE/CVF Conference on Computer Vision and Pattern Recognition}}. \bibinfo{pages}{1622--1631}.
\newblock


\bibitem[Wang et~al\mbox{.}(2025)]%
        {wang2025text}
\bibfield{author}{\bibinfo{person}{Ruiyu Wang}, \bibinfo{person}{Yu Yuan}, \bibinfo{person}{Shizhao Sun}, {and} \bibinfo{person}{Jiang Bian}.} \bibinfo{year}{2025}\natexlab{}.
\newblock \showarticletitle{Text-to-CAD Generation Through Infusing Visual Feedback in Large Language Models}.
\newblock \bibinfo{journal}{\emph{arXiv preprint arXiv:2501.19054}} (\bibinfo{year}{2025}).
\newblock


\bibitem[Wang et~al\mbox{.}(2020)]%
        {wang2020pie}
\bibfield{author}{\bibinfo{person}{Xiaogang Wang}, \bibinfo{person}{Yuelang Xu}, \bibinfo{person}{Kai Xu}, \bibinfo{person}{Andrea Tagliasacchi}, \bibinfo{person}{Bin Zhou}, \bibinfo{person}{Ali Mahdavi-Amiri}, {and} \bibinfo{person}{Hao Zhang}.} \bibinfo{year}{2020}\natexlab{}.
\newblock \showarticletitle{Pie-net: Parametric inference of point cloud edges}.
\newblock \bibinfo{journal}{\emph{Advances in neural information processing systems}}  \bibinfo{volume}{33} (\bibinfo{year}{2020}), \bibinfo{pages}{20167--20178}.
\newblock


\bibitem[Willis et~al\mbox{.}(2022)]%
        {willis2022joinable}
\bibfield{author}{\bibinfo{person}{Karl~DD Willis}, \bibinfo{person}{Pradeep~Kumar Jayaraman}, \bibinfo{person}{Hang Chu}, \bibinfo{person}{Yunsheng Tian}, \bibinfo{person}{Yifei Li}, \bibinfo{person}{Daniele Grandi}, \bibinfo{person}{Aditya Sanghi}, \bibinfo{person}{Linh Tran}, \bibinfo{person}{Joseph~G Lambourne}, \bibinfo{person}{Armando Solar-Lezama}, {et~al\mbox{.}}} \bibinfo{year}{2022}\natexlab{}.
\newblock \showarticletitle{Joinable: Learning bottom-up assembly of parametric cad joints}. In \bibinfo{booktitle}{\emph{Proceedings of the IEEE/CVF conference on computer vision and pattern recognition}}. \bibinfo{pages}{15849--15860}.
\newblock


\bibitem[Willis et~al\mbox{.}(2021a)]%
        {willis2021engineering}
\bibfield{author}{\bibinfo{person}{Karl~DD Willis}, \bibinfo{person}{Pradeep~Kumar Jayaraman}, \bibinfo{person}{Joseph~G Lambourne}, \bibinfo{person}{Hang Chu}, {and} \bibinfo{person}{Yewen Pu}.} \bibinfo{year}{2021}\natexlab{a}.
\newblock \showarticletitle{Engineering sketch generation for computer-aided design}. In \bibinfo{booktitle}{\emph{Proceedings of the IEEE/CVF conference on computer vision and pattern recognition}}. \bibinfo{pages}{2105--2114}.
\newblock


\bibitem[Willis et~al\mbox{.}(2021b)]%
        {willis2021fusion}
\bibfield{author}{\bibinfo{person}{Karl~DD Willis}, \bibinfo{person}{Yewen Pu}, \bibinfo{person}{Jieliang Luo}, \bibinfo{person}{Hang Chu}, \bibinfo{person}{Tao Du}, \bibinfo{person}{Joseph~G Lambourne}, \bibinfo{person}{Armando Solar-Lezama}, {and} \bibinfo{person}{Wojciech Matusik}.} \bibinfo{year}{2021}\natexlab{b}.
\newblock \showarticletitle{Fusion 360 gallery: A dataset and environment for programmatic cad construction from human design sequences}.
\newblock \bibinfo{journal}{\emph{ACM Transactions on Graphics (TOG)}} \bibinfo{volume}{40}, \bibinfo{number}{4} (\bibinfo{year}{2021}), \bibinfo{pages}{1--24}.
\newblock


\bibitem[Wu et~al\mbox{.}(2024)]%
        {wu2024diffusion}
\bibfield{author}{\bibinfo{person}{Hongjie Wu}, \bibinfo{person}{Linchao He}, \bibinfo{person}{Mingqin Zhang}, \bibinfo{person}{Dongdong Chen}, \bibinfo{person}{Kunming Luo}, \bibinfo{person}{Mengting Luo}, \bibinfo{person}{Ji-Zhe Zhou}, \bibinfo{person}{Hu Chen}, {and} \bibinfo{person}{Jiancheng Lv}.} \bibinfo{year}{2024}\natexlab{}.
\newblock \showarticletitle{Diffusion Posterior Proximal Sampling for Image Restoration}. In \bibinfo{booktitle}{\emph{Proceedings of the 32nd ACM International Conference on Multimedia}}. \bibinfo{pages}{214--223}.
\newblock


\bibitem[Wu et~al\mbox{.}(2025)]%
        {wu2025enhancing}
\bibfield{author}{\bibinfo{person}{Hongjie Wu}, \bibinfo{person}{Mingqin Zhang}, \bibinfo{person}{Linchao He}, \bibinfo{person}{Ji-Zhe Zhou}, {and} \bibinfo{person}{Jiancheng Lv}.} \bibinfo{year}{2025}\natexlab{}.
\newblock \bibinfo{title}{Enhancing Diffusion Model Stability for Image Restoration via Gradient Management}.
\newblock
\showeprint[arxiv]{2507.06656}~[cs.CV]
\urldef\tempurl%
\url{https://arxiv.org/abs/2507.06656}
\showURL{%
\tempurl}


\bibitem[Wu et~al\mbox{.}(2021)]%
        {wu2021deepcad}
\bibfield{author}{\bibinfo{person}{Rundi Wu}, \bibinfo{person}{Chang Xiao}, {and} \bibinfo{person}{Changxi Zheng}.} \bibinfo{year}{2021}\natexlab{}.
\newblock \showarticletitle{Deepcad: A deep generative network for computer-aided design models}. In \bibinfo{booktitle}{\emph{Proceedings of the IEEE/CVF International Conference on Computer Vision}}. \bibinfo{pages}{6772--6782}.
\newblock


\bibitem[Xu et~al\mbox{.}(2023)]%
        {xu2023hierarchical}
\bibfield{author}{\bibinfo{person}{Xiang Xu}, \bibinfo{person}{Pradeep~Kumar Jayaraman}, \bibinfo{person}{Joseph~G Lambourne}, \bibinfo{person}{Karl~DD Willis}, {and} \bibinfo{person}{Yasutaka Furukawa}.} \bibinfo{year}{2023}\natexlab{}.
\newblock \showarticletitle{Hierarchical neural coding for controllable cad model generation}.
\newblock \bibinfo{journal}{\emph{arXiv preprint arXiv:2307.00149}} (\bibinfo{year}{2023}).
\newblock


\bibitem[Xu et~al\mbox{.}(2021)]%
        {xu2021inferring}
\bibfield{author}{\bibinfo{person}{Xianghao Xu}, \bibinfo{person}{Wenzhe Peng}, \bibinfo{person}{Chin-Yi Cheng}, \bibinfo{person}{Karl~DD Willis}, {and} \bibinfo{person}{Daniel Ritchie}.} \bibinfo{year}{2021}\natexlab{}.
\newblock \showarticletitle{Inferring cad modeling sequences using zone graphs}. In \bibinfo{booktitle}{\emph{Proceedings of the IEEE/CVF conference on computer vision and pattern recognition}}. \bibinfo{pages}{6062--6070}.
\newblock


\bibitem[Xu et~al\mbox{.}(2022)]%
        {xu2022skexgen}
\bibfield{author}{\bibinfo{person}{Xiang Xu}, \bibinfo{person}{Karl~DD Willis}, \bibinfo{person}{Joseph~G Lambourne}, \bibinfo{person}{Chin-Yi Cheng}, \bibinfo{person}{Pradeep~Kumar Jayaraman}, {and} \bibinfo{person}{Yasutaka Furukawa}.} \bibinfo{year}{2022}\natexlab{}.
\newblock \showarticletitle{SkexGen: Autoregressive Generation of CAD Construction Sequences with Disentangled Codebooks}. In \bibinfo{booktitle}{\emph{International Conference on Machine Learning}}. PMLR, \bibinfo{pages}{24698--24724}.
\newblock


\bibitem[Xue et~al\mbox{.}(2024)]%
        {xue2024unifying}
\bibfield{author}{\bibinfo{person}{Kaiwen Xue}, \bibinfo{person}{Yuhao Zhou}, \bibinfo{person}{Shen Nie}, \bibinfo{person}{Xu Min}, \bibinfo{person}{Xiaolu Zhang}, \bibinfo{person}{Jun Zhou}, {and} \bibinfo{person}{Chongxuan Li}.} \bibinfo{year}{2024}\natexlab{}.
\newblock \showarticletitle{Unifying bayesian flow networks and diffusion models through stochastic differential equations}.
\newblock \bibinfo{journal}{\emph{arXiv preprint arXiv:2404.15766}} (\bibinfo{year}{2024}).
\newblock


\bibitem[You et~al\mbox{.}(2024)]%
        {you2024img2cad}
\bibfield{author}{\bibinfo{person}{Yang You}, \bibinfo{person}{Mikaela~Angelina Uy}, \bibinfo{person}{Jiaqi Han}, \bibinfo{person}{Rahul Thomas}, \bibinfo{person}{Haotong Zhang}, \bibinfo{person}{Suya You}, {and} \bibinfo{person}{Leonidas Guibas}.} \bibinfo{year}{2024}\natexlab{}.
\newblock \showarticletitle{Img2cad: Reverse engineering 3d cad models from images through vlm-assisted conditional factorization}.
\newblock \bibinfo{journal}{\emph{arXiv preprint arXiv:2408.01437}} (\bibinfo{year}{2024}).
\newblock


\bibitem[Yu et~al\mbox{.}(2025)]%
        {yu2025knowledge}
\bibfield{author}{\bibinfo{person}{Jun-Lin Yu}, \bibinfo{person}{Cong Zhou}, \bibinfo{person}{Xiang-Li Ning}, \bibinfo{person}{Jun Mou}, \bibinfo{person}{Fan-Bo Meng}, \bibinfo{person}{Jing-Wei Wu}, \bibinfo{person}{Yi-Ting Chen}, \bibinfo{person}{Biao-Dan Tang}, \bibinfo{person}{Xiang-Gen Liu}, {and} \bibinfo{person}{Guo-Bo Li}.} \bibinfo{year}{2025}\natexlab{}.
\newblock \showarticletitle{Knowledge-guided diffusion model for 3D ligand-pharmacophore mapping}.
\newblock \bibinfo{journal}{\emph{Nature Communications}} \bibinfo{volume}{16}, \bibinfo{number}{1} (\bibinfo{year}{2025}), \bibinfo{pages}{2269}.
\newblock


\bibitem[Zhang et~al\mbox{.}(2025)]%
        {zhang2025diffusion}
\bibfield{author}{\bibinfo{person}{Aijia Zhang}, \bibinfo{person}{Weiqiang Jia}, \bibinfo{person}{Qiang Zou}, \bibinfo{person}{Yixiong Feng}, \bibinfo{person}{Xiaoxiang Wei}, {and} \bibinfo{person}{Ye Zhang}.} \bibinfo{year}{2025}\natexlab{}.
\newblock \showarticletitle{Diffusion-CAD: Controllable Diffusion Model for Generating Computer-Aided Design Models}.
\newblock \bibinfo{journal}{\emph{IEEE Transactions on Visualization and Computer Graphics}} (\bibinfo{year}{2025}).
\newblock


\end{thebibliography}

\clearpage
\appendix
\section{Theoretical Proofs}

\subsection{Proof of Theorem~\ref{cor:cbu}}\label{appendix:p1}
\begin{proof}
The Bayesian derivation begins with the standard decomposition:
\begin{equation}
    p(\boldsymbol{\theta}_i | \boldsymbol{\theta}_{i-1}, \mathbf{x}, \alpha, \mathcal{C}) 
    = \frac{p(\boldsymbol{\theta}_i | \boldsymbol{\theta}_{i-1}, \mathbf{x}, \alpha)  p(\mathcal{C} | \boldsymbol{\theta}_i, \boldsymbol{\theta}_{i-1}, \mathbf{x}, \alpha)  p(\boldsymbol{\theta}_{i-1}, \mathbf{x}, \alpha)}{p(\mathcal{C}, \boldsymbol{\theta}_{i-1}, \mathbf{x}, \alpha)}.
\end{equation}

Under the assumed conditional independence, the likelihood term simplifies as follows:
\begin{equation}
    \begin{aligned}
    p(\mathcal{C} \mid \boldsymbol{\theta}_i, \boldsymbol{\theta}_{i-1}, \mathbf{x}, \alpha)
    &= \frac{p(\boldsymbol{\theta}_{i-1}, \mathbf{x}, \alpha, \mathcal{C} \mid \boldsymbol{\theta}_i) \cdot p(\boldsymbol{\theta}_i)}{p(\boldsymbol{\theta}_i, \boldsymbol{\theta}_{i-1}, \mathbf{x}, \alpha)} \\
    &= \frac{p(\boldsymbol{\theta}_{i-1}, \mathbf{x}, \alpha \mid \boldsymbol{\theta}_i) \cdot p(\mathcal{C} \mid \boldsymbol{\theta}_i) \cdot p(\boldsymbol{\theta}_i)}{p(\boldsymbol{\theta}_i, \boldsymbol{\theta}_{i-1}, \mathbf{x}, \alpha)} \\
    &= p(\mathcal{C} \mid \boldsymbol{\theta}_i).
    \end{aligned}
\end{equation}

Furthermore, the hyperparameter $\alpha$ is conditionally independent of the target property $\mathcal{C}$ given $\boldsymbol{\theta}_i$, which gives:
\begin{equation}
p(\mathcal{C} \mid \boldsymbol{\theta}_i) = p(\mathcal{C} \mid \boldsymbol{\theta}_i, \alpha).
\end{equation}

Substituting the above into the original expression yields the desired factorization:
\begin{equation}
    \begin{aligned}
    p(\boldsymbol{\theta}_i \mid \boldsymbol{\theta}_{i-1}, \mathbf{x}, \alpha, \mathcal{C}) 
    &= \frac{p(\boldsymbol{\theta}_i \mid \boldsymbol{\theta}_{i-1}, \mathbf{x}, \alpha) \cdot p(\mathcal{C} \mid \boldsymbol{\theta}_i) \cdot p(\boldsymbol{\theta}_{i-1}, \mathbf{x}, \alpha)}{p(\boldsymbol{\theta}_{i-1}, \mathbf{x}, \alpha) \cdot p(\mathcal{C})} \\
    &\propto p(\boldsymbol{\theta}_i \mid \boldsymbol{\theta}_{i-1}, \mathbf{x}, \alpha) \cdot p(\mathcal{C} \mid \boldsymbol{\theta}_i, \alpha),
    \end{aligned}
\end{equation}
which completes the proof.
\end{proof}


\subsection{Proof of Theorem~\ref{thm:p2}}\label{app:p2}
\begin{proof}
We present the theoretical foundation for the constraint-aware estimators introduced in Section~\ref{subsec:cbf}.

\begin{lemma}\label{lemma:measure}
Let $\mathbf{x}$ denote an input sequence, and let $\mathcal{C} = f(\mathbf{x})$ be a property deterministically induced by $\mathbf{x}$. Then, for any parameter setting $(\boldsymbol{\theta}_i, \alpha)$, the likelihood of observing $\mathcal{C}$ is equivalent to the likelihood of observing $\mathbf{x}$ under the constraint $f(\mathbf{x}) = \mathcal{C}$:
\begin{equation}
    p(\mathcal{C} \mid \boldsymbol{\theta}_i, \alpha) \equiv p(\mathbf{x} \mid \boldsymbol{\theta}_i, \alpha) \quad \text{s.t.} \quad f(\mathbf{x}) = \mathcal{C}.
\end{equation}
\end{lemma}

\begin{proof}
Since $\mathcal{C}$ is a deterministic function of $\mathbf{x}$, the conditional distribution reduces to:
\begin{equation}
    p(\mathcal{C} \mid \mathbf{x}, \boldsymbol{\theta}_i, \alpha) = \delta(\mathcal{C} - f(\mathbf{x})).
\end{equation}

The joint distribution over $(\mathbf{x}, \mathcal{C})$ then factorizes as:
\begin{equation}
    p(\mathbf{x}, \mathcal{C} \mid \boldsymbol{\theta}_i, \alpha) = \delta(\mathcal{C} - f(\mathbf{x})) \cdot p(\mathbf{x} \mid \boldsymbol{\theta}_i, \alpha).
\end{equation}
Marginalizing out $\mathbf{x}$ yields:
\begin{equation}
    p(\mathcal{C} \mid \boldsymbol{\theta}_i, \alpha) = \int \delta(\mathcal{C} - f(\mathbf{x})) \cdot p(\mathbf{x} \mid \boldsymbol{\theta}_i, \alpha) \, d\mathbf{x}.
\end{equation}
This integral collects the total probability mass over all $\mathbf{x}$ such that $f(\mathbf{x}) = \mathcal{C}$. Hence, the two likelihoods are equivalent under the constraint.
\end{proof}

\begin{theorem}\label{thm:decomp}
The distribution over CAD sequences conditioned on current parameters admits the following factorization:
\begin{equation}
p(\mathbf{x}\mid\boldsymbol{\theta}_i, \alpha) \propto p(\mathbf{x}) \cdot p_S(\boldsymbol{\theta}_i\mid\mathbf{x}, \alpha),
\end{equation}
where $p_S$ denotes the sender distribution in the Bayesian flow framework.
\end{theorem}

\begin{proof}
Starting from Bayes’ rule:
\begin{equation}
p(\mathbf{x}\mid\boldsymbol{\theta}_i, \alpha) = \frac{p_S(\boldsymbol{\theta}_i\mid\mathbf{x}, \alpha) \cdot p(\mathbf{x})}{\mathbb{E}_{p(\mathbf{x}')}[p_S(\boldsymbol{\theta}_i\mid\mathbf{x}', \alpha)]}.
\end{equation}
The denominator is a normalizing constant:
\begin{equation}
Z(\boldsymbol{\theta}_i, \alpha) = \int p_S(\boldsymbol{\theta}_i\mid\mathbf{x}', \alpha) \cdot p(\mathbf{x}') \, d\mathbf{x}'.
\end{equation}
This yields the proportional form:
\begin{equation}
p(\mathbf{x}\mid\boldsymbol{\theta}_i, \alpha) \propto p(\mathbf{x}) \cdot p_S(\boldsymbol{\theta}_i\mid\mathbf{x}, \alpha),
\end{equation}
which underlies the weighted sampling formulation used in our constraint-aware estimator.
\end{proof}

We now derive the expressions for the mean and variance of $p(\mathcal{C} \mid \boldsymbol{\theta}_i, \alpha)$ using Lemma~\ref{lemma:measure} and Theorem~\ref{thm:decomp}.

By definition:
\begin{equation}
\mu(\boldsymbol{\theta}_i, \alpha) = \mathbb{E}_{p(\mathcal{C} \mid \boldsymbol{\theta}_i, \alpha)}[\mathcal{C}] = \mathbb{E}_{p(\mathbf{x} \mid \boldsymbol{\theta}_i, \alpha)}[f(\mathbf{x})].
\end{equation}

Substituting from Theorem~\ref{thm:decomp}:
\begin{equation}
\mathbb{E}_{p(\mathbf{x} \mid \boldsymbol{\theta}_i, \alpha)}[f(\mathbf{x})] = \frac{1}{Z} \int p(\mathbf{x}) \cdot p_S(\boldsymbol{\theta}_i \mid \mathbf{x}, \alpha) \cdot f(\mathbf{x}) \, d\mathbf{x}.
\end{equation}

This motivates the weighted empirical estimator:
\begin{equation}
\mu(\boldsymbol{\theta}_i, \alpha) = \frac{1}{Z} \sum_{(\mathbf{x}, \mathcal{C}) \sim \mathcal{D}} p_S(\boldsymbol{\theta}_i \mid \mathbf{x}, \alpha) \cdot \mathcal{C},
\end{equation}
where
\begin{equation}
Z = \sum_{(\mathbf{x}, \mathcal{C}) \sim \mathcal{D}} p_S(\boldsymbol{\theta}_i \mid \mathbf{x}, \alpha).
\end{equation}

Similarly, the variance is:
\begin{align}
\Sigma(\boldsymbol{\theta}_i, \alpha) 
&= \mathbb{E}_{p(\mathcal{C} \mid \boldsymbol{\theta}_i, \alpha)}\left[ (\mathcal{C} - \mu)^2 \right] \nonumber \\
&= \frac{1}{Z} \sum_{(\mathbf{x}, \mathcal{C}) \sim \mathcal{D}} 
p_S(\boldsymbol{\theta}_i \mid \mathbf{x}, \alpha) 
\left[\mathcal{C} - \mu(\boldsymbol{\theta}_i, \alpha)\right]^2.
\end{align}

\end{proof}


\section{Sampling Algorithms}
\label{app:algo}

We summarize the sampling procedures for both baseline and proposed models. Algorithm~\ref{alg:bfn} outlines the standard BFN sampling process~\cite{graves2023bayesian}, where each step samples a discrete class from a learned categorical distribution (parameterized by \texttt{OUTPUT\_DISTRIBUTION}, Alg.~\ref{alg:outdist}), injects Gaussian noise, and applies a normalized exponential transformation to update the latent parameters $\boldsymbol{\theta}$. Our proposed method, TGBFN (Alg.~\ref{alg:tgbfn}), extends this framework by incorporating Unbiased Bayesian Inference (UBI) for parallel sampling, Guided Bayesian Flow (GBF) for condition-aware refinement via $p_\psi(\mathcal{C}|\cdot)$, and Calibrated Distribution Proximal (CDP) for moment-preserving quantization. At each step, $m$ parallel samples are drawn, reweighted, and aggregated to balance uncertainty and constraint alignment, forming a coherent and efficient generative process for constrained CAD synthesis.

\begin{algorithm}[htbp]
\caption {OUTPUT\_DISTRIBUTION}
\label{alg:outdist}
\SetKwInOut{KwInput}{Require}
\KwInput{
$\boldsymbol{\theta}\in[0,1]^{KD}$, $t \in [0,1]$, target condition $\mathcal{C}$
}
\For{$d\in\{1,D\}$}{
    $p_O^{(d)}(\cdot \mid \boldsymbol{\theta},t,\mathcal{C})\leftarrow\text{softmax}(\Psi^{(d)}(\boldsymbol{\theta},t,\mathcal{C}))$
}
\Return{$p_O(\cdot \mid \boldsymbol{\theta,t,\mathcal{C}})$}\;
\end{algorithm}

\begin{algorithm}[htbp]
\caption {BFN Sampling~\cite{graves2023bayesian}}
\label{alg:bfn}
\SetKwInOut{KwInput}{Require}
\KwInput{
$\beta(1) \in \mathbb{R}_+$, number of steps $n \in \mathbb{N}$, number of classes $K \in \mathbb{N}$
}
$\boldsymbol{\theta} \leftarrow ( \frac{\mathbf{1}}{\mathbf{K}} )$\;
\For{$i=1$ \textbf{to} $n$}{
    $t \leftarrow \frac{i-1}{n}$\;
    $\mathbf{k} \sim \text{OUTPUT\_DISTRIBUTION}(\boldsymbol{\theta},t)$\;
    $\alpha\sim\beta(1)(\frac{2i-1}{n^2})$\;
    $\mathbf{y}\sim\mathcal{N}(\alpha(K\mathbf{e_k-1}), \alpha K\mathbf{I})$\;
    $\boldsymbol{\theta}\leftarrow \frac{e^\mathbf{y}\boldsymbol{\theta}}{\sum_k(e^\mathbf{y}\boldsymbol{\theta})_k}$\;
}
$\mathbf{k} \sim \text{OUTPUT\_DISTRIBUTION}(\boldsymbol{\theta},t)$\;
\Return{$\mathbf{k}$}\;
\end{algorithm}

\begin{algorithm}[htbp]
\caption {TGBFN sampling}
\label{alg:tgbfn}
\SetKwInOut{KwInput}{Require}
\KwInput{
target condition $\mathcal{C}$, $\beta(1) \in \mathbb{R}_+$, number of steps $n \in \mathbb{N}$, number of classes $K \in \mathbb{N}$, the trained conditional guided network $p_{\psi}$
}
$\boldsymbol{\theta}_0 \leftarrow ( \frac{\mathbf{1}}{\mathbf{K}} )$\;
\For{$i=1$ \textbf{to} $n$}{
    $t \leftarrow \frac{i-1}{n}$\;
    \For{$j=1$ \textbf{to} $m$}{
        \For{$h=1$ \textbf{to} $H$}{
            $\mathbf{k}_h \sim \text{OUTPUT\_DISTRIBUTION}(\boldsymbol{\theta}_{i-1},t,\mathcal{C})$\;
        }
        $\boldsymbol{k} = \mathrm{NEAREST\_CATEGORY}\left( \frac{1}{H}\sum_h \boldsymbol{k}_h \right)$\;
        $\alpha\sim\beta(1)(\frac{2i-1}{n^2})$\;
        $\mathbf{y}_j\sim\mathcal{N}(\alpha(K\mathbf{e_k-1}), \alpha K\mathbf{I})$\;
        $\boldsymbol{\theta}_i^{(j)}\leftarrow \frac{e^\mathbf{y}\boldsymbol{\theta}_{i-1}}{\sum_k(e^\mathbf{y}\boldsymbol{\theta}_{i-1})_k}$\;
        $\boldsymbol{\theta}_i^{(j)}\leftarrow\boldsymbol{\theta}_i^{(j)} \cdot p_{\psi}(\mathcal{C}\mid\boldsymbol{\theta}_i^{(j)},\alpha)$\;  
    }
    $\boldsymbol{\theta}_i\leftarrow\frac{\sum_j\boldsymbol{\theta}_i^{(j)}}{\sum_jp_{\psi}(\mathcal{C}\mid\boldsymbol{\theta}_i^{(j)},\alpha)}$
}
$\mathbf{k} \sim \text{OUTPUT\_DISTRIBUTION}(\boldsymbol{\theta}_n,t,\mathcal{C})$\;
\Return{$\mathbf{k}$}\;
\end{algorithm}


\section{Ablation under Single-Property Constraints}
\label{app:ablation-single}

To isolate the contributions of individual components in our framework, we perform an ablation study under single-property supervision, where models are conditioned solely on either surface area or volume. The results, summarized in Table~\ref{tab:ablation_area}, reveal that incorporating the Unbiased Bayesian Inference (UBI) module yields the most significant performance improvement over the base model for both constraints, markedly reducing errors and increasing correlation. The addition of Guided Bayesian Flow (GBF) further enhances performance, particularly by improving alignment with geometric targets. Finally, the Calibrated Distribution Estimation (CDE) component offers marginal but consistent gains in accuracy and correlation, culminating in the best overall model performance.

\begin{table}[htbp]
\centering
\caption{Ablation study under single-property supervision. We evaluate the contribution of each component by progressively adding them to the base model. Results are reported separately for surface area and volume constraints.}
\vspace{-0.3cm}
\resizebox{1.0\linewidth}{!}{
\begin{tabular}{@{}l||ccc@{}}
\toprule
Model & {MSE $\downarrow$} & {MAE $\downarrow$} & {PCC $\uparrow$} \\
\toprule
\multicolumn{4}{c}{\textbf{Surface Area Constraints}} \\
\toprule
Base Model   & 5.1052 & 0.7854 & 0.7865 \\
BFN + UBI    & 0.5705 & 0.4541 & 0.9341 \\
BFN + UBI + GBF  & 0.4666 & 0.3472 & 0.9488 \\
\midrule
\rowcolor{gray!15}
Base + UBI \& GBF \& CDE~(Ours) & \textbf{0.3503} & \textbf{0.3067} & \textbf{0.9587} \\
\bottomrule
\multicolumn{4}{c}{\textbf{Volume Constraints}} \\
\toprule
Base Model       & 0.2188 & 0.1030 & 0.7748 \\
BFN + UBI    & 0.0172 & 0.0584 & 0.9267 \\
BFN + UBI + GBF  & 0.0160 & 0.0574 & 0.9231 \\
\midrule
\rowcolor{gray!15}
Base + UBI \& GBF \& CDE~(Ours) & \textbf{0.0158} & \textbf{0.0529} & \textbf{0.9342} \\
\bottomrule
\end{tabular}
}
\label{tab:ablation_area}
\end{table}

These findings confirm that each component contributes meaningfully under both constraint types and that the full model achieves the most reliable and accurate constraint-aware CAD generation.


\section{Implementation Details} \label{app:id}

\subsection{Training Configuration}
The skeleton network $\phi$ is trained with a learning rate of $1\times10^{-6}$ and a batch size of 128 on a single RTX 4090 GPU for approximately 12 days. The conditional guidance network $\psi$ is trained separately with a learning rate of $1\times10^{-5}$ and the same batch size on a single RTX 4090 GPU for 3 days.

\subsection{Network Architecture}
We employ a decoder-style Transformer as the denoising network $\phi$ to model the distribution over CAD token sequences. The architecture comprises 21 Transformer layers, each with 16 attention heads and a hidden size of 1024. Input tokens are embedded into a 1024-dimensional space and augmented with learnable positional encodings via $\texttt{nn.Embedding}(L, 1024)$, where $L$ is the sequence length.

The network follows the standard Transformer configuration, incorporating Layer Normalization, residual connections, and Dropout. Unlike autoregressive models, we do not apply causal masking, allowing full bidirectional attention across the token sequence. At each step, the model predicts a categorical distribution over a fixed vocabulary of 263 tokens, enabling full-sequence refinement during denoising-based generation.


\begin{figure*}[t]
    \centering
    \includegraphics[width=1.0\linewidth]{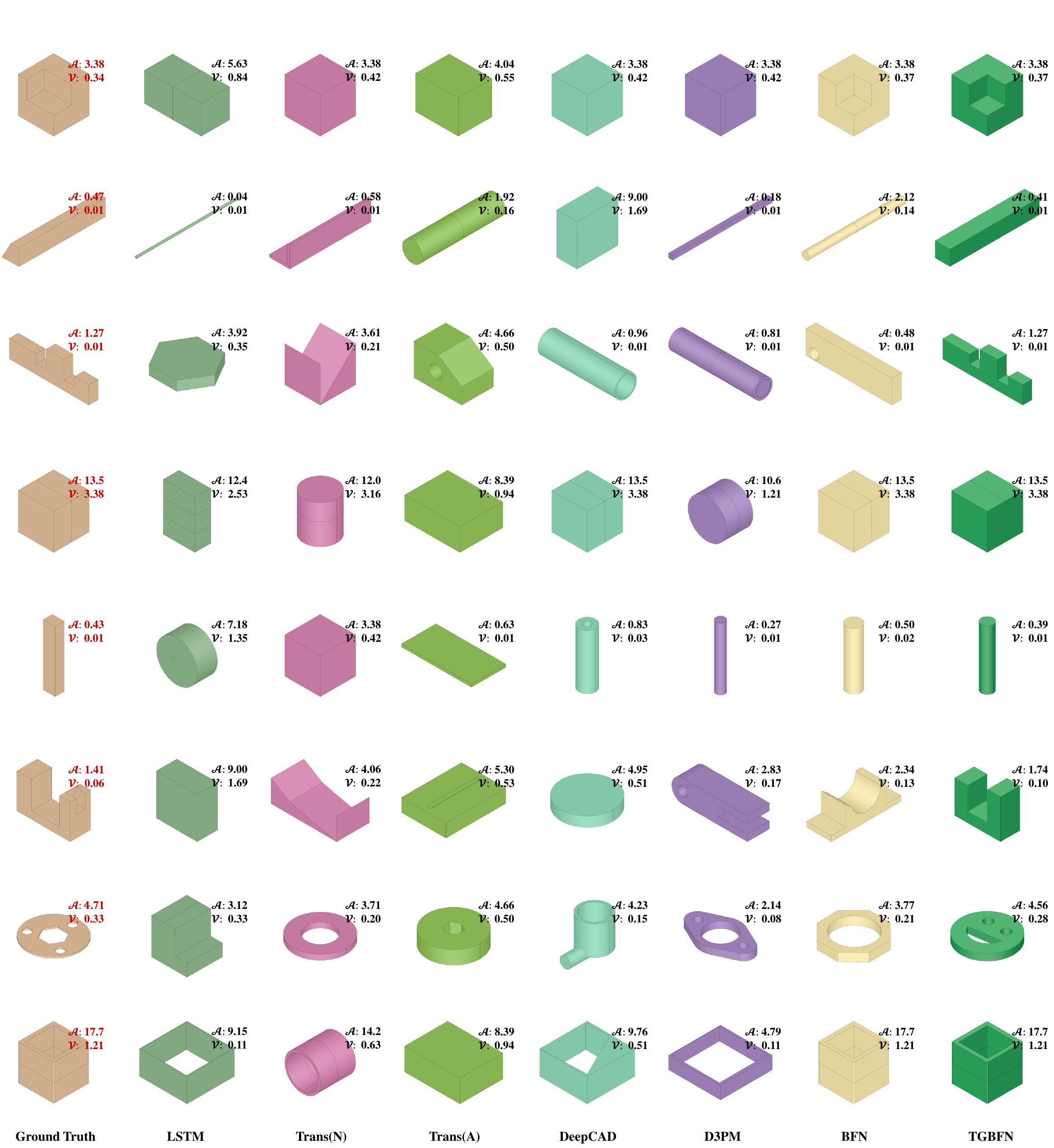}
    \caption{Additional case studies of CAD sequence generation under surface area and volume constraints.}
    \label{fig:appcs}
\end{figure*}

\end{document}